\documentclass{article}

\usepackage{microtype}
\usepackage{graphicx}
\usepackage{subfigure}
\usepackage{booktabs} % for professional tables

\usepackage{hyperref}

\usepackage[accepted]{icml2021}

\usepackage{amsmath} % added by qiuwei
\usepackage{rotating} % added by qiuwei
\usepackage{adjustbox} % added by qiuwei
\usepackage{pdflscape} % qiuwei 

\usepackage{amsthm} % added by qiuwei
\usepackage{textcomp} % added by qiuwei
\usepackage{amsfonts} % added by qiuwei
\usepackage{bm} % added by qiuwei
\usepackage{amsthm} % added by qiuwei
\usepackage{mathrsfs} % added by qiuwei
\usepackage{thmtools, thm-restate} % added by qiuwei
\usepackage{graphicx} % added by qiuwei
\usepackage{subfigure} % added by qiuwei
\usepackage{chngcntr} % added by qiuwei
\usepackage{makecell} % added by qiuwei
\usepackage{wrapfig}% added by qiuwei 

\usepackage[T1]{fontenc} % added by qiuwei
\usepackage{mathtools}

\DeclarePairedDelimiter\floor{\lfloor}{\rfloor}

\declaretheorem[name=Definition]{definition}

\icmltitlerunning{RMIX: Learning Risk-Sensitive Policies for Cooperative Reinforcement Learning Agents}

\begin{document}

\twocolumn[
\icmltitle{RMIX: Learning Risk-Sensitive Policies for Cooperative \\ Reinforcement Learning Agents}



\icmlsetsymbol{equal}{*}

\begin{icmlauthorlist}
\icmlauthor{Wei Qiu}{ntu}
\icmlauthor{Xinrun Wang}{ntu}
\icmlauthor{Runsheng Yu}{hkust}
\icmlauthor{Xu He}{ntu}
\icmlauthor{Rundong Wang}{ntu}
\icmlauthor{Bo An}{ntu}
\icmlauthor{Svetlana Obraztsova}{ntu}
\icmlauthor{Zinovi Rabinovich}{ntu}
\end{icmlauthorlist}

\icmlaffiliation{ntu}{School of Computer Science and Engineering, Nanyang Technological University, Singapore}
\icmlaffiliation{hkust}{Department of Computer Science and Engineering, Hong Kong University of Science and Technology, Hong Kong SAR, China}

\icmlcorrespondingauthor{Wei Qiu}{qiuw0008@e.ntu.edu.sg}

\icmlkeywords{Machine Learning, ICML}

\vskip 0.3in
]



\printAffiliationsAndNotice{}  % leave blank if no need to mention equal contribution

\begin{abstract}
Current value-based multi-agent reinforcement learning methods optimize individual Q values to guide individuals' behaviours via centralized training with decentralized execution (CTDE). However, such expected, i.e., risk-neutral, Q value is not sufficient even with CTDE due to the randomness of rewards and the uncertainty in environments, which causes the failure of these methods to train coordinating agents in complex environments. To address these issues, we propose RMIX, a novel cooperative MARL method with the Conditional Value at Risk (CVaR) measure over the learned distributions of individuals' Q values. Specifically, we first learn the return distributions of individuals to analytically calculate CVaR for decentralized execution. Then, to handle the temporal nature of the stochastic outcomes during executions, we propose a dynamic risk level predictor for risk level tuning. Finally, we optimize the CVaR policies with CVaR values used to estimate the target in TD error during centralized training and the CVaR values are used as auxiliary local rewards to update the local distribution via Quantile Regression loss. Empirically, we show that our method significantly outperforms state-of-the-art methods on challenging StarCraft II tasks, demonstrating enhanced coordination and improved sample efficiency.
\end{abstract}

\section{Introduction}\label{Intro}
% \textcolor{purple}{Distinguish MAS and MARL. }
Reinforcement learning (RL) has made remarkable advances in many domains, including arcade video games~\citep{mnih2015human}, complex continuous robot control~\citep{lillicrap2015continuous} and the game of Go~\citep{silver2017mastering1}. Recently, many researchers put their efforts to extend the RL methods into multi-agent systems (MASs), such as urban systems~\citep{singh2020hierarchical}, coordination of robot swarms~\citep{huttenrauch2017guided} and real-time strategy (RTS) video games~\citep{vinyals2019grandmaster}. Centralized training with decentralized execution (CTDE)~\citep{oliehoek2008optimal} has drawn enormous attention via training policies of each agent with access to global trajectories in a centralized way and executing actions given only the local observations of each agent in a decentralized way. Empowered by CTDE, several multi-agent RL (MARL) methods, including value-based and policy gradient-based, are proposed~\citep{foerster2017counterfactual,sunehag2017value,rashid2018qmix,son2019qtran}. These MARL methods propose decomposition techniques to factorize the global Q value either by structural constraints or by estimating state-values or inter-agent weights to conduct the global Q value estimation. Among these methods, VDN~\citep{sunehag2017value} and QMIX~\citep{rashid2018qmix} are representative methods that use additivity and monotonicity structure constraints, respectively. With relaxed structural constraints, QTRAN~\citep{son2019qtran} guarantees a more general factorization than VDN and QMIX. Recently, Qatten~\citep{yang2020qatten} incorporates multi-head attention to represent the global values.

Despite the merits, most of these works focus on decomposing the global Q value into individual Q values with different constraints and network architectures, but ignore the fact that such the expected, i.e., risk-neutral, Q value is not sufficient as optimistic actions executed by some agents can impede the team coordination such as imprudent actions in hostage rescue operations, which causes the failure of these methods to train coordinating agents in complex environments. Specifically, these methods only learn the expected values over returns~\citep{rashid2018qmix} and do not handle the high variance caused by events with extremely high/low rewards to agents but at small probabilities, which cause the inaccurate/insufficient estimations of the future returns. Therefore, instead of expected values, learning distributions of future returns, i.e., Q values, are more useful for agents to make decisions. Recently, QR-MIX~\citep{hu2020qr} decomposes the estimated joint return distribution~\citep{bellemare2017distributional,dabney2018IQN} into individual Q values. However, the policies in QR-MIX are still individual Q values. Even further, given that the environment is nonstationary from the perspective of each agent, decision-making over the agent's return distribution takes events of potential return into account, which makes agents able to address uncertainties in the environment compared with simply taking the expected values for execution. However, current MARL methods do not extensively investigate these aspects. 

Motivated by the previous reasons, we intend to extend the risk-sensitive\footnote{``Risk'' refers to the uncertainty of future outcomes~\citep{dabney2018IQN}.} RL~\citep{chow2014algorithms,keramati2019being,zhang2020cautious} to MARL settings, where risk-sensitive RL optimizes policies with a risk measure, such as variance, power formula measure value at risk (VaR) and conditional value at risk (CVaR). 
% These measures can be used as constraints for policy optimization or policies for execution. 
Among these risk measures, CVaR has been gaining popularity due to both theoretical and computational advantages~\citep{rockafellar2002conditional,ruszczynski2010risk}.
However, there are two main obstacles: i) most of the previous works focus on risk-neutral or static risk level in the single-agent settings, ignoring the randomness of reward and the temporal structure of agents’ trajectories~\citep{dabney2018IQN,tang2019worst,ma2020distributional,keramati2019being}; ii) many methods use risk measures over Q values for policy execution without getting the risk measure values used in policy optimization in temporal difference (TD) learning, which causes the global value factorization on expected individual values to have sub-optimal behaviours in MARL. 

In this paper, we propose RMIX, a novel cooperative MARL method with CVaR over the learned distributions of individuals' Q values. Specifically, our contributions are in three folds: (i) We first learn the return distributions of individuals by using Dirac Delta functions in order to analytically calculate CVaR for decentralized execution. The resulting CVaR values at each time step are used as policies for each agent via $\arg\max$ operation; (ii) We then propose a dynamic risk level predictor for CVaR calculation to handle the temporal nature of stochastic outcomes as well as tune the risk level during executions. The dynamic risk level predictor measures the discrepancy between the embedding of current individual return distributions and the embedding of historical return distributions. The dynamic risk levels are agent-specific and observation-wise; (iii)  As our method focuses on optimizing the CVaR policies via CTDE, we finally optimize CVaR policies with CVaR values as target estimators in TD error via centralized training and CVaR values are used as auxiliary local rewards to update local return distributions via Quantile Regression loss. These also allow our method to achieve temporally extended exploration and enhanced temporal coordination, which are key to solving complex multi-agent tasks. Empirically, we show that RMIX significantly outperforms state-of-the-art methods on many challenging StarCraft II\textsuperscript{TM}\footnote{StarCraft II is a trademark of Blizzard Entertainment, Inc.} tasks, demonstrating enhanced coordination in many \emph{symmetric} \& \emph{asymmetric} and \emph{homogeneous} \& \emph{heterogeneous} scenarios and revealing high sample efficiency. To the best of our knowledge, our work is the \emph{first} attempt to investigate cooperative MARL with risk-sensitive policies under the Dec-POMDP framework. 

\textbf{Related Works.} CTDE~\citep{oliehoek2008optimal} has drawn enormous attention via training policies of each agent with access to global trajectories in a centralized way and executing actions given only the local observations of each agent in a decentralized way. However, current MARL methods~\citep{lowe2017multi,foerster2017counterfactual,sunehag2017value,rashid2018qmix,son2019qtran,hu2020qr} neglect the limited representation of agent values, thus failing to consider the problem of random cost underlying the nonstationarity of the environment, a.k.a risk-sensitive learning. Recent advances in distributional RL~\citep{bellemare2017distributional,dabney2018distributional} focus on learning distribution over returns. However, these works still focus on either risk-neutral settings or with static risk level in single-agent setting. \citet{chow2014algorithms} considered the mean-CVaR optimization problem in MDPs and proposed policy gradient with CVaR, and ~\citet{garcia2015comprehensive} presented a survey on safe RL, which initiated the research on utilizing risk measures in RL~\citep{garcia2015comprehensive,tamar2015policy,tang2019worst,hiraoka2019learning,majumdar2020should,keramati2019being,ma2020distributional}. However, these works focus on single-agent settings. The merit of CVaR in optimization of MARL has yet to be investigated.

\section{Preliminaries}\label{background}
In this section, we provide the notation and the basic notions we will use in the following. We consider the probability space $\left(\Omega, \mathcal{F}, \text{Pr} \right)$, where $\Omega$ is the set of outcomes (sample space), $\mathcal{F}$ is a $\sigma$-algebra over $\Omega$ representing the set of events, and $\text{Pr}$ is the set of probability distributions. Given a set $\mathcal{X}$, we denote with $\mathscr{P}(\mathcal{X})$ the set of all probability measures over $\mathcal{X}$.

\textbf{Dec-POMDP.} A fully cooperative MARL problem can be described as a \textit{decentralised partially observable Markov decision process} (Dec-POMDP)~\citep{oliehoek2016concise} which can be formulated as a tuple $\mathcal{M} = \langle\mathcal{S},\mathcal{U},\mathcal{P},R, \Upsilon,O,\mathcal{N},\gamma\rangle$, where $\bm{s} \in \mathcal{S}$ denotes the state of the environment. Each agent $i \in \mathcal{N} := \{1,...,N\} $ chooses an action $u_i \in \mathcal{U}$ at each time step, giving rise to a joint action vector, $\bm{u} := [u_i]_{i=1}^N \in \mathcal{U}^N$. $\mathcal{P}(\bm{s}'|\bm{s},\bm{u}):\mathcal{S} \times \mathcal{U}^N \times \mathcal{S} \mapsto \mathcal{P}(\mathcal{S})$ is a Markovian transition function. Every agent shares the same joint reward function $R(\bm{s},\bm{u}): \mathcal{S} \times \mathcal{U}^N \mapsto \mathcal{R} $, and $\gamma \in [0,1)$ is the discount factor. Due to \textit{partial observability},   each agent has individual partial observation $\upsilon \in \Upsilon$, according to the observation function $O(\bm{s},i): \mathcal{S} \times \mathcal{N} \mapsto \Upsilon.$ Each agent learns its own policy $\pi_i(u_i|\tau_i) : \mathcal{T} \times \mathcal{U} \mapsto [0,1]$ given its action-observation history $\tau_i \in \mathcal{T} := (\Upsilon \times \mathcal{U})$.

\textbf{CVaR.} CVaR is a coherent risk measure and enjoys computational properties~\citep{rockafellar2002conditional} that are derived for loss distributions in discreet decision-making in finance. It gains popularity in various engineering and finance applications.
\begin{wrapfigure}{r}{0.2\textwidth}
\vspace{-0.3cm}
\includegraphics[width=\linewidth]{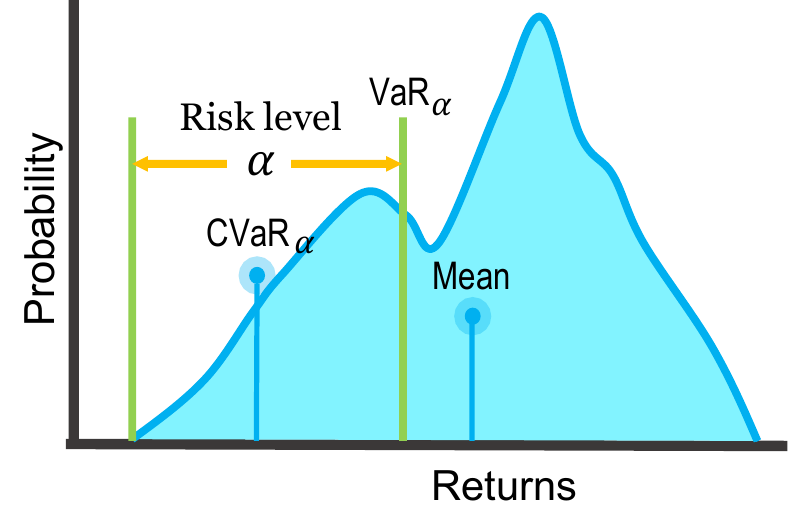}
\vspace{-0.6cm}
\caption{CVaR}\label{fig:cvar}
\vspace{-0.4cm}
\end{wrapfigure}
CVaR (Figure \ref{fig:cvar}) is the expectation of values that are less equal than the $\alpha$-percentile value of the distribution over returns. Formally, let $X \in \mathcal{X}$ be a bounded random variable with cumulative distribution function $F(x)= \mathscr{P} \left [X\leq x \right ]$ and the inverse CDF is $F^{-1}(u) = \inf\{ x : F(x) \geq u\}$. CVaR at level $\alpha \in (0,1]$ of a random variable $X$ is then defined as $\text{CVaR}_\alpha(X) := \sup_\nu \left\{\nu -    \frac{1}{\alpha}\mathbb{E}[{(\nu - X)^+]}\right\}$~\citep{rockafellar2000optimization} when $X$ is a discrete random variable. Correspondingly, $\text{CVaR}_\alpha(X) = \mathbb{E}_{X \sim F}
\left [X | X \leq F^{-1}(\alpha) \right ]$ \citep{acerbi2002coherence} when $X$ has a continuous distribution. The $\alpha$-percentile value is value at risk (VaR). For ease of notation, we write $\text{CVaR}$ as $\text{CVaR}_\alpha(F)$.

\begin{figure}[t]
    % \vspace{-0.2cm}
    \begin{center}
        \includegraphics[width=0.4\textwidth]{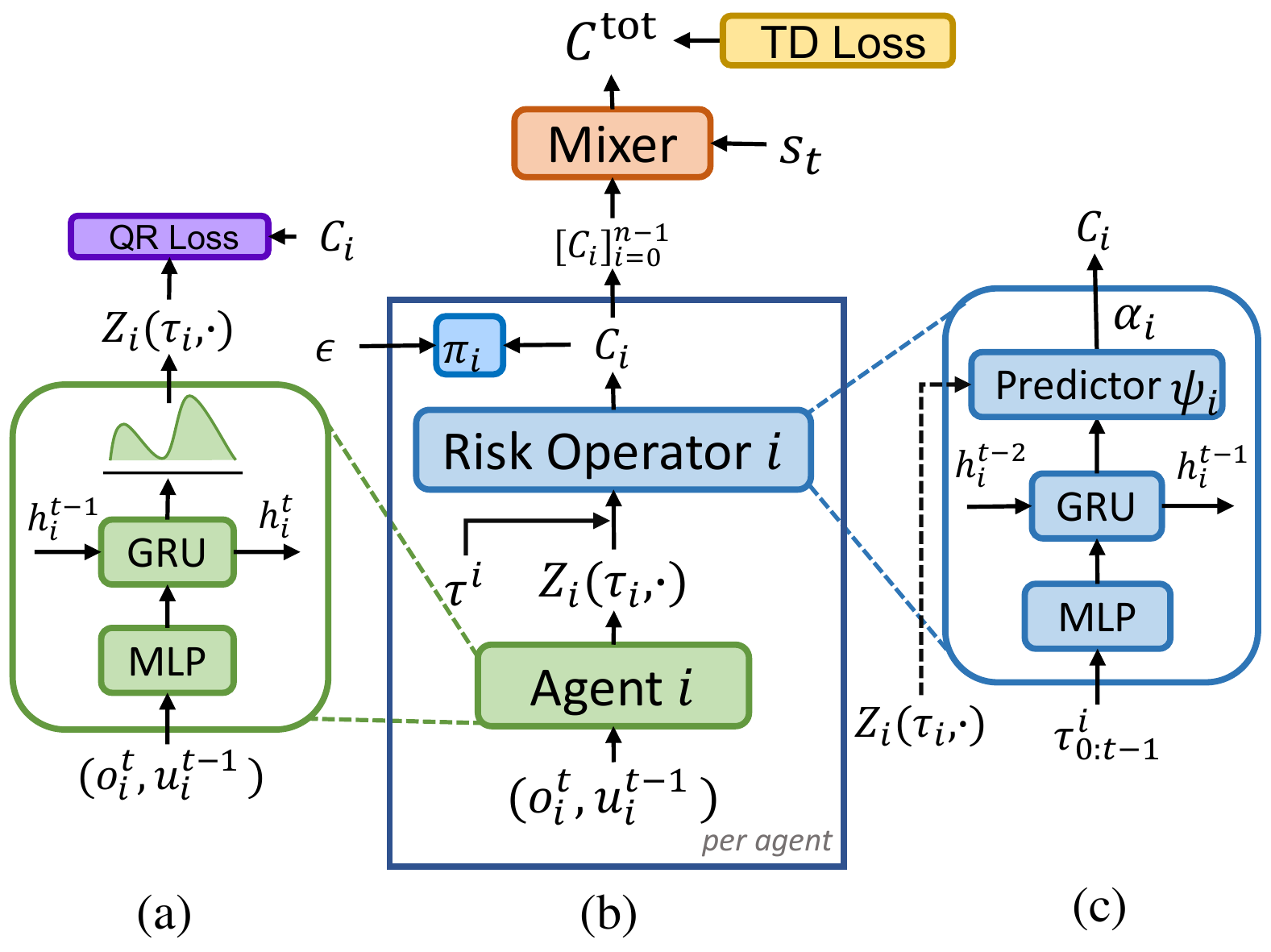}
        % \vspace{-0.4cm}
        \caption{Our framework (dotted arrow indicates that gradients are blocked during training). (a) Agent's policy network. (b) The overall architecture (agent network and mixer). (c)  Risk operator. Each agent $i$ applies an individual risk operator $\Pi_{\alpha_i}$ on its return distribution $Z_i(\cdot, \cdot)$ to calculate $C_i(\cdot, \cdot, \cdot)$ for execution given risk level $\alpha_i$ predicted by the dynamic risk level predictor $\psi_i$. $\{C_i(\cdot, \cdot, \cdot)\}_{i=1}^{N}$ are fed into the mixer for centralized training. }\label{fig:rmix_arch}
    \end{center}
	\vspace{-0.5cm}
\end{figure}

\textbf{Risk-sensitive RL.} Risk-sensitive RL uses risk criteria over policy/value, which is a sub-field of the Safety RL~\citep{garcia2015comprehensive}. \citeauthor{von1947theory} proposed the expected utility theory where a decision policy behaves as though it is maximizing the expected value of some utility functions. The condition is satisfied when the decision policy is consistent and has a particular set of four axioms. This is the most pervasive notion of risk-sensitivity. A policy maximizing a linear utility function is called \textit{risk}-\textit{neutral}, whereas concave or convex utility functions give rise to \textit{risk}-\textit{averse} or \textit{risk}-\textit{seeking} policies, respectively. Many measures are used in RL such as CVaR~\citep{chow2015risk,dabney2018IQN} and power formula~\citep{dabney2018IQN}. However, few works have been done in MARL and it cannot be easily extended. Our work fills this gap.

\section{Methodology}\label{method}

\subsection{CVaR of Return Distribution}
\label{sec:cvar_estimation}
In this section, we describe how we estimate the CVaR value. The value of CVaR can be either estimated through sampling or computed from the parameterized return distribution~\citep{rockafellar2002conditional}. However, the sampling method is usually computationally expensive~\citep{tang2019worst}. Therefore, we let each agent learn a return distribution parameterized by a mixture of Dirac Delta ($\delta$) functions~\footnote{The Dirac Delta is a \textit{Generalized function} in the theory of distributions and not a function given the properties of it. We use the name \textit{Dirac Delta function} by convention.}, which is demonstrated to be highly expressive and computationally efficient~\citep{bellemare2017distributional}.
For convenience, we provide the definition of the \textit{Generalized Return Probability Density Function (PDF)}.
\begin{definition} (Generalized Return PDF). For a discrete random variable $R \in \left[-R_{\operatorname{max}}, R_{\operatorname{max}} \right]$ and probability mass function (PMF) $\mathscr{P}(R=r_k)$, where $r_k \in \left[-R_{\operatorname{max}}, R_{\operatorname{max}} \right]$, we define the generalized return PDF as: $f_{R}(r)=\sum\nolimits_{r_k \in R} \mathscr{P}_R(r_k)\delta(r-r_k)$. Note that for any $r_k \in R$, the probability of $R=r_k$ is given by the coefficient of the corresponding $\delta$ function, $\delta (r - r_k)$.
\end{definition}
We define the parameterized return distribution of each agent $i$ at time step $t$ as:
\begin{equation}
    Z^{t}_{i} (\tau_{i}, u^{t-1}_{i}) = \sum\nolimits_{j=1}^{M} \mathscr{P}_j(\tau_{i}, u^{t-1}_{i})\delta_j (\tau_{i}, u^{t-1}_{i})
\end{equation}
where $M$ is the number of Dirac Delta functions. $\delta_j (\tau_{i}, u^{t-1}_{i})$ is the $j$-th Dirac Delta function and indicates the estimated value which can be parameterized by neural networks in practice. $\mathscr{P}_j(\tau_{i}, u^{t-1}_{i})$ is the corresponding probability of the estimated value given local observations and actions. $\tau_i$ and $u^{t-1}_i$ are trajectories (up to that timestep) and actions of agent $i$, respectively. 
With the individual return distribution $Z^{t}_{i} (\tau_{i}, u^{t-1}_{i}) \in \mathcal{Z}$ and cumulative distribution function (CDF) $F_{Z_{i} (\tau_{i}, u^{t-1}_{i})}$, we define the CVaR operator $\Pi_{\alpha_{i}}$, at a risk level $\alpha_{i}$ ($\alpha_{i} \in (0, 1]$ and $i \in \mathcal{A}$), over return as\footnote{We will omit $t$ in the rest of the paper for notation brevity.} $C^{t}_{i}(\tau_i, u^{t-1}_i, \alpha_i)=\Pi_{\alpha^{t}_{i}} Z^{t}_{i} (\tau_{i}, u^{t-1}_{i}) := \operatorname{CVaR}_{\alpha^{t}_{i}} (F_{Z^{t}_{i} (\tau_{i}, u^{t-1}_{i})})$
where $C \in \mathcal{C}$. As we use CVaR on return distributions, it corresponds to risk-neutrality (expectation, $\alpha_i=1$) and indicates the improving degree of risk-aversion ($\alpha_i \rightarrow 0$). $\operatorname{CVaR}_{\alpha_{i}}$ can be estimated in a nonparametric way given ordering of Dirac Delta functions $\left\{\delta_{j}\right\}_{j=1}^{m}$~\citep{kolla2019concentration} by leveraging the individual distribution:
\begin{equation}\label{eq:cvar_risk}
    \operatorname{CVaR}_{\alpha_{i}}=\sum\nolimits_{j=1}^{m} \mathscr{P}_j \delta_{j} \bm{1}\left\{\delta_{j} \leq \hat{v}_{m, \alpha_{i}}\right\},
\end{equation}
where $\bm{1}\{\cdot\}$ is the indicator function and $\hat{v}_{m, \alpha_{i}}$ is estimated value at risk from $\hat{v}_{m, \alpha_{i}}= \floor{\delta_{m(1-\alpha_{i})}}$ with $\floor{\cdot}$ being floor function. This is a closed-form formulation and can be easily implemented in practice. The optimal action of agent $i$ can be calculated via $\arg\max_{u_i} C_{i} (\tau_{i}, u^{t-1}_{i}, \alpha_{i})$. We will introduce it in detail in Sec. \ref{sec:dynamic_risk_level}.
\begin{figure}[t]
    \vspace{-0.2cm}
    \begin{center}
        \includegraphics[width=0.24\textwidth]{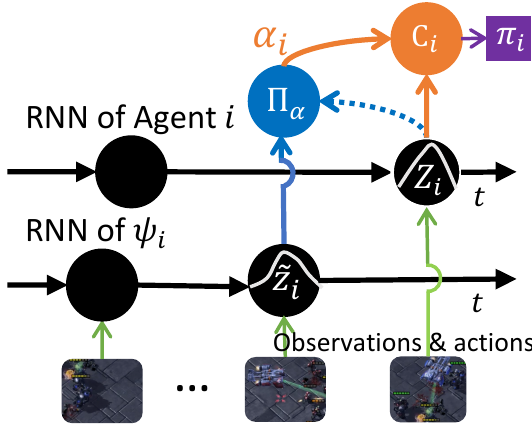}
        % \vspace{-0.4cm}
        \caption{Agent architecture.}\label{fig:marl_agent}
    \end{center}
	\vspace{-0.5cm}
\end{figure}

\subsection{Risk Level Predictor}
\label{sec:dynamic_risk_level}

The values of risk levels, i.e., $\alpha_{i}$, $i\in\mathcal{A}$, are important for the agents to make decisions. Most of the previous works take a fixed value of risk level and do not take into account any temporal structure of agents' trajectories, which is hard to tune the best risk level and may impede centralized training in the evolving multi-agent environments. Therefore, we propose the dynamic risk level predictor, which determines the risk levels of agents by explicitly taking into account the temporal nature of the stochastic outcomes, to alleviate time-consistency issue~\citep{ruszczynski2010risk,iancu2015tight} and stabilize the centralized training. Specifically, we represent the risk operator $\Pi_{\alpha}$ by a deep neural network, which calculates the CVaR value with predicted dynamic risk level $\alpha$ over the return distribution.

We show the architecture of agent $i$ in Figure \ref{fig:marl_agent} (agent network and risk operator with risk level predictor $\psi_i$, as shown in Figure \ref{fig:rmix_arch}) and illustrate how $\psi_i$ works with  agent $i$ for CVaR calculation in practice in Figure \ref{fig:risk_predictor_main}. As depicted in Figure \ref{fig:risk_predictor_main}, at time step $t$, the agent's return distribution is $Z_i$ and its historical return distribution is $\Tilde{Z}_i$. Then we conduct the inner product to measure the discrepancy between the embedding of individual return distribution $f_{\operatorname{emb}}(Z_i)$ and the embedding of past trajectory $\phi_{i}(\tau^{0:t-1}_{i}, u^{t-1}_{i})$ modeled by GRU~\citep{chung2014empirical}. 
We discretize the risk level range into $K$ even ranges for the purpose of computing. The $k$-th dynamic risk level $\alpha^{k}_{i}$ is output from $\psi_i$ and the probability of $\alpha^{k}_{i}$ is defined as:
\begin{equation}\label{eq:alpha_cal}
    \mathscr{P}(\alpha^{k}_{i}) = \frac{\exp \left(\left<f_{\operatorname{emb}}(Z_i)^{k}, \phi_{i}^{k}\right>\right)}{\sum\nolimits_{k^{\prime}=0}^{K-1} \exp \left(\left<f_{\operatorname{emb}}(Z_i)^{k^{\prime}}, \phi_{i}^{k^{\prime}} \right>\right)}.
\end{equation}
Then we get the $k \in [1,\dots, K]$ with the maximal probability by $\arg\max$ and normalize it into $(0, 1]$, thus $\alpha_i=k/K$. The prediction risk level ${\alpha}_i$ is a scalar value and it is converted into a $K$-dimensional mask vector where the first $\floor{\alpha_i\times K}$ items are one and the rest are zero. This mask vector is used to calculate the CVaR value (Eqn. \ref{eq:cvar_risk}) of each action-return distribution that contains $K$ Dirac functions. Finally, we obtain $C_i$ and  the policy $\pi_i$ as illustrated in Figure \ref{fig:marl_agent}. During training, $f_{\operatorname{emb}_{i}}$ updates its weights and the gradients of $f_{\operatorname{emb}_{i}}$ are blocked (the dotted arrow in Figure \ref{fig:marl_agent}) in order to prevent changing the weights of the network of agent $i$.
\begin{figure}[t]
    \vspace{-0.2cm}
    \begin{center}
        \includegraphics[width=0.32\textwidth]{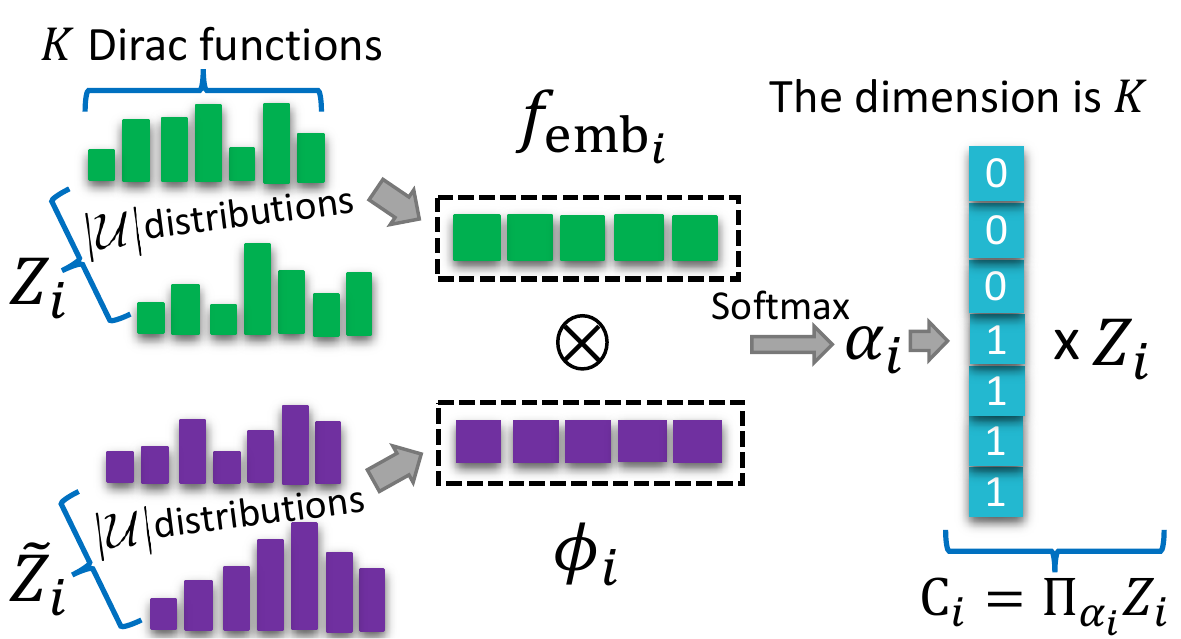}
        % \vspace{-0.4cm}
        \caption{Risk level predictor $\psi_i$.}\label{fig:risk_predictor_main}
    \end{center}
	\vspace{-0.4cm}
\end{figure}
We note that the predictor differs from the attention network used in previous works~\citep{iqbal2019actor,yang2020qatten} because the agent's current return distribution and its return distribution of previous time step are separate inputs of their embeddings and there is no \textit{key}, \textit{query} and \textit{value} weight matrices. The dynamic risk level predictors allow agents to determine the risk level dynamically based on historical return distributions and it is a hyperparameter (e.g., $\alpha$) tuning strategy as well.

\subsection{Training}\label{sec:centralized_training}
We train RMIX via addressing the two challenging issues: credit assignment and return distribution learning. 

% We introduce the centralized training of RMIX. and RDN. RDN is a simple variant of VDN by utilizing individual CVaR values. 
\textbf{Loss Function.} As there is only a global reward signal and agents have no access to individuals' reward, we first utilize the monotonic mixing network ($f_{m}$) from QMIX to do Credit assignment. $f_{m}$ enforces a monotonicity constraint on the relationship between $C^{\operatorname{tot}}$ and each $C_{i}$ for RMIX:
\begin{equation}
    \frac{\partial C^{\operatorname{tot}}}{\partial C_{i}} \geq 0, \forall i \in\{1,2, \ldots, N\},
\end{equation}
where $C^{\operatorname{tot}} = f_m (C_{1}(\cdot, \cdot, \cdot), \cdots, C_{N}(\cdot, \cdot, \cdot) )$ and $C_{i}(\tau_{i}, u_{i}, \alpha_i)$ is the individual CVaR value of agent $i$. To ease the confusion, the $C^{\operatorname{tot}}$ is not the global CVaR value as modeling global return distribution as well as local distribution are challenging with the credit assignment issue in Dec-POMDP problems, and the risk level values are locally decided and used during training. Then, to maximize the CVaR value of each agent, we define the risk-sensitive Bellman operator $\mathcal{T}$:
\begin{equation}\label{eq:cvar_bellman}
    \mathcal{T} C^{\operatorname{tot}} (\boldsymbol{s}, \boldsymbol{u}) := \mathbb{E}\left[ R(\boldsymbol{s}, \boldsymbol{u}) + \gamma \max_{\boldsymbol{u}^{\prime}} C^{\operatorname{tot}} (\boldsymbol{s}^{\prime}, \boldsymbol{u}^{\prime}) \right]
\end{equation}
% provides a guarantee for attaining the fixed points in optimality.
The risk-sensitive Bellman operator $\mathcal{T}$ operates on the $C^{\operatorname{tot}}$ and the reward, which can be proved to be a contracting operation, as shown in Proposition~\ref{propContract}. 
\begin{restatable}{proposition}{PropContract}\label{propContract}
$\mathcal{T}: \mathcal{C} \mapsto \mathcal{C}$ is a $\gamma$-contraction.
\end{restatable}
Therefore, we can leverage the TD learning~\citep{sutton2018reinforcement} to train RMIX. Following the CTDE paradigm, we define our TD loss:
\begin{equation}
\label{eq:td_loss}
\mathcal{L}_{\Pi}(\theta) := \mathbb{E}_{\mathcal{D}^{\prime} \sim \mathcal{D} }\left[({y}_{t}^{\operatorname{tot}} -  C^{\operatorname{tot}}\left(\boldsymbol{s}_{t}, \boldsymbol{u}_{t}\right))^{2}\right]
\end{equation}
where ${y}_{t}^{\operatorname{tot}}=\left(r_{t}+\gamma \max_{\boldsymbol{u}^{\prime}} C_{\bar{\theta}}^{\operatorname{tot}}\left(\boldsymbol{s}_{t+1}, \boldsymbol{u}^{\prime}\right)\right)$, and $({y}_{t}^{\operatorname{tot}} -  C_{\theta}^{\operatorname{tot}}\left(\boldsymbol{s}_{t}, \boldsymbol{u}_{t}\right))$ is our CVaR TD error for updating CVaR values. $\theta$ is the parameters of $C^{\operatorname{tot}}$ which can be modeled by a deep neural network and $\bar{\theta}$ indicates the parameters of the target network which is periodically copied from $\theta$ for stabilizing training~\citep{mnih2015human}. While training, gradients from $Z_i$ are blocked to avoid changing the weights of the agents' network from the dynamic risk level predictor.

\textbf{Local Return Distribution Learning.} The CVaR estimation relies on accurately updating the local return distribution and the update is non-trivial. However, unlike many deep learning~\citep{goodfellow2014generative} and distributional RL methods~\citep{bellemare2017distributional,dabney2018distributional} where the label and local reward signal are accessible, in our problem, the exact rewards for each agent are unknown, which is very common in real world problems. To address this issue, we first consider CVaR values as dummy rewards of each agent due to its  property of modeling the potential loss of return and then leverage the Quantile Regression (QR) loss used in Distributional RL~\citep{dabney2018distributional} to explicitly update the local distribution decentrally. More concretely, QR aims to estimate the quantiles of the return distribution by minimizing the quantile regression loss between $Z_{i} (\tau_i, u_i)$ and its target distribution $\hat{Z}_{i} (\tau_i, u_i) = C_i(\tau_i, u_i, \alpha_{i}) + \gamma Z_i ({\tau_i}^{\prime}, {u_i}^{\prime}) $. Formally, the quantile distribution is represented by a set of quantiles $\tau_{j}=\frac{j}{K}$ and the quantile regression loss for Q network is defined as
\begin{equation}\label{eq:qr_loss}
    \mathcal{L}_{QR} = \frac{1}{N} \sum\nolimits_{i=1}^{N}  \sum\nolimits_{j=1}^{K} \mathbb{E}_{\hat{Z}_{i} \sim Z_{i}} [ \rho_{\tau_{j}} (\hat{Z}_{i} - Z_{i}) ]
\end{equation}
where $\rho_{\tau} (\nu) = \nu (\tau - \bm{1}\{\nu < 0\})$. To eliminate cuspid in $\rho_{\tau}$ which could limit performance when using non-linear function approximation, quantile Huber loss is used as the loss function. The quantile Huber loss is defined as $\rho_{\tau} (\nu) = \mathcal{L}_{\kappa}(\nu) |\tau - \bm{1}\{\nu < 0\}|$
where $\mathcal{L}_{\kappa}(\nu)$ is defined as:
\begin{align}
    \mathcal{L}_{\kappa}(\nu)= 
    \begin{cases}
          \frac{1}{2}\nu^2, & \text{if} |\nu| \geq \kappa,\\
          \kappa(|\nu| - \frac{1}{2} \kappa), & \text{otherwise.}
    \end{cases}
\end{align}
Note that Risk-sensitive RL and Distributional RL are two orthogonal research directions as discussed in Sec. \ref{Intro} and \ref{background}.

\textbf{Training.} Finally, we train RMIX in an end-to-end manner where each agent shares a single agent network and a risk predictor network to solve the lazy-agent issue~\citep{sunehag2017value}. $\psi_i$ is trained together with the agent network via the loss defined in Eqn. \ref{eq:td_loss}. During training, $f_{\operatorname{emb}_{i}}$ updates its weights while the gradients of $f_{\operatorname{emb}_{i}}$ are blocked in order to prevent changing the weights of the return distribution in agent $i$. In fact, agents only use CVaR values for execution and the risk level predictor only predicts the $\alpha$; thus the increased network capacity is mainly from the local return distribution and the CVaR operator. Our framework is flexible and can be easily used in many cooperative MARL methods. We present our framework in Figure \ref{fig:rmix_arch} and the pseudo code is shown in Appendix. All proofs are provided in Appendix.

% \ref{appendResults} 
\begin{figure*}[ht]
% \vspace{-0.45cm}
\begin{center}
    \includegraphics[width=1\textwidth]{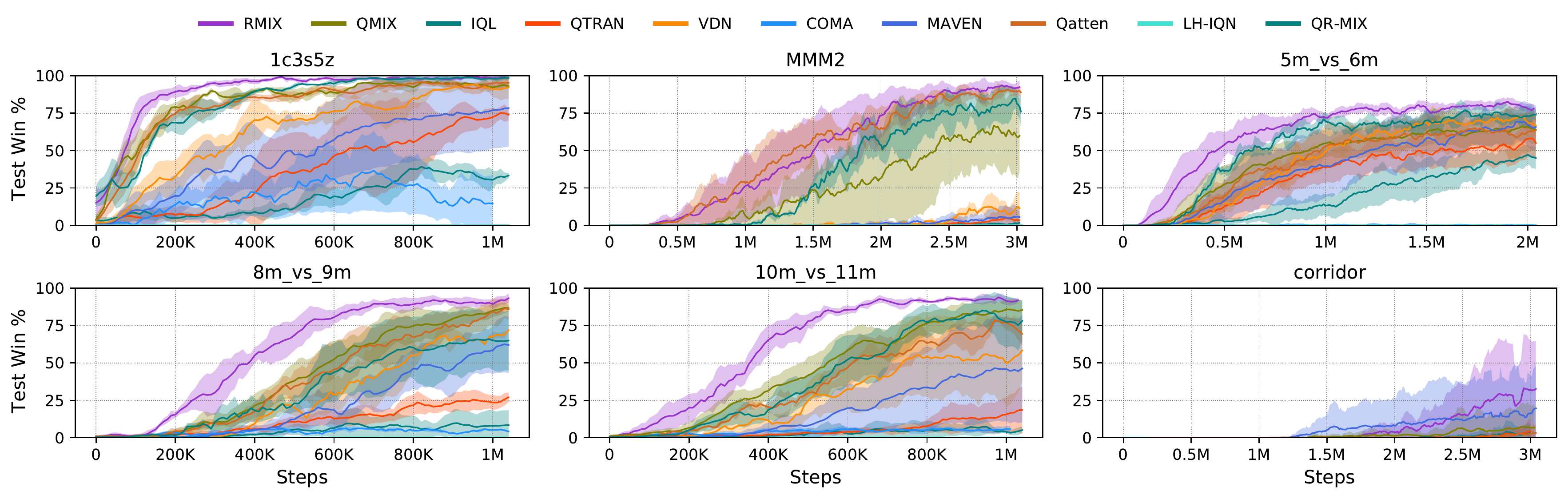}
    \vspace{-0.8cm}
    \caption{Test Winning rates for six scenarios. More results can be found in Appendix.}\label{res_1}
 \vspace{-0.3cm}
 \end{center}
\end{figure*}
\section{Theoretical Analysis}
Insightfully, our proposed method can be categorized into an overestimation reduction perspective which has been investigated in single-agent domain~\citep{thrun1993issues,hasselt2010double,lan2019maxmin,chen2021randomized}. Intuitively, during minimizing $\mathcal{L}_{\Pi}(\theta)$ and policy execution, we can consider CVaR implementation as calculating the mean over k-minimum $\delta$ values of $Z_i$. It motivates us to analyse our method's overestimation reduction property.

In single-agent cases, the overestimation bias occurs since the target $\max_{a^{\prime}} Q\left(s_{t+1}, a^{\prime}\right)$ is used in the Q-learning update. This maximum operator over target estimations is likely to be skewed towards an overestimate as Q is an approximation which is possibly higher than the true value for one or more of the actions. In multi-agent scenarios, for example StarCraft II, the primary goal for each agent is to survive (maintain positive health values) and win the game. Overestimation on high return values might lead to agents suffering defeat early-on in the game.

Formally, in MARL, we characterize the relation between the estimation error, the in-target minimization parameter $\alpha$ and the number of Dirac functions, $M$, which consist of the return distribution. We follow the theoretical framework introduced in~\citep{thrun1993issues} and extended in~\citep{chen2021randomized}. More concretely, let $C^{\operatorname{tot}}(\boldsymbol{s}, \boldsymbol{u})-Q (\boldsymbol{s}, \boldsymbol{u})$ be the pre-update estimation bias for the output $C^{\operatorname{tot}}$ with the chosen individual CVaR values, where $Q(\boldsymbol{s}, \boldsymbol{u})$ is the ground-truth Q-value. We are interested in how the bias changes after an update, and how this change is affected by risk level $\alpha$. The post-update estimation bias, which is the difference between two different targets, can be defined as:
\begin{equation}
\begin{aligned}
\Psi_{\boldsymbol{\alpha}} &\overset{\Delta}{=} r + \gamma \max_{\boldsymbol{u}^{\prime}} C^{\operatorname{tot}}\left(\boldsymbol{s}^{\prime}, \boldsymbol{u}^{\prime}\right) - \left(r + \gamma \max_{\boldsymbol{u}^{\prime}} Q^{\operatorname{tot}}\left(\boldsymbol{s}^{\prime}, \boldsymbol{u}^{\prime}\right)\right) \\ &= \gamma \left(\max_{\boldsymbol{u}^{\prime}} C^{\operatorname{tot}}\left(\boldsymbol{s}^{\prime}, \boldsymbol{u}^{\prime}\right)-\max _{\boldsymbol{u}^{\prime}} Q^{\operatorname{tot}}\left(\boldsymbol{s}^{\prime}, \boldsymbol{u}^{\prime}\right)\right) \nonumber
\end{aligned}
\end{equation}
where $\boldsymbol{\alpha}=\{\alpha_i\}_{i=1}^{N}$ and
$Q\left(\boldsymbol{s}^{\prime}, \boldsymbol{u}^{\prime}\right)$ is output of the centralized mixing network with individuals' $Q_i$ as input. Note that due to the zero-mean assumption, the expected pre-update estimation bias is $\mathbb{E}[
C^{\operatorname{tot}}(\boldsymbol{s}, \boldsymbol{u}) - Q(\boldsymbol{s}, \boldsymbol{u})] = 0$ by following \citet{thrun1993issues} and \citet{lan2019maxmin}. Thus if $\mathbb{E}[\Psi_{\alpha}] > 0$, the expected post-update bias is positive and there is a tendency for over-estimation accumulation; and if $\mathbb{E}[\Psi_{\alpha}]<0$, there is a tendency for under-estimation accumulation.
\begin{restatable}{theorem}{PropOverEst}\label{propOverEst}
We summarize the following properties:\\
(1) Given $\boldsymbol{\alpha}_1$ and $\boldsymbol{\alpha}_2$ where $\alpha$ values in each set are identical,  $\mathbb{E}[\Psi_{\boldsymbol{\alpha}_1}] \leq \mathbb{E}[\Psi_{\boldsymbol{\alpha}_2}]$ for $0 < \alpha_1 \leq \alpha_2 \leq 1$, $\alpha_1 \in \boldsymbol{\alpha}_1$ and $\alpha_2 \in \boldsymbol{\alpha}_2$. \\
(2) $\exists \alpha_i \in (0, 1]$ and $\boldsymbol{\alpha}=\{\alpha_i\}_{i=1}^{N}$, $\mathbb{E}[\Psi_{\boldsymbol{\alpha}}] < 0$.
\end{restatable}
Theorem \ref{propOverEst} implies that we can control the $\mathbb{E}[\Psi_{\boldsymbol{\alpha}}]$, bringing it from above zero (overestimation) to under zero (underestimation) by decreasing $\alpha$. Thus, we can control the post-update bias with the risk level $\alpha$ and boost the RMIX training on scenarios where overestimation can lead to failure of cooperation. In the next section, we will present the empirical results.

\section{Experiments}\label{experiments}
We empirically evaluate our method on various StarCraft II (SCII) scenarios where agents coordinate and defeat the opponent. Especially, we are interested in the robust cooperation between agents and agents' learned risk-sensitive policies in complex \emph{asymmetric} and \emph{homogeneous}/\emph{heterogeneous} scenarios.
Additional introduction of scenarios and results are in Appendix.
\begin{figure}[ht]
% \vspace{-0.1cm}
\begin{center}
  \subfigure[5m\_vs\_6m]{\includegraphics[height=2.5cm,width=.23\textwidth]{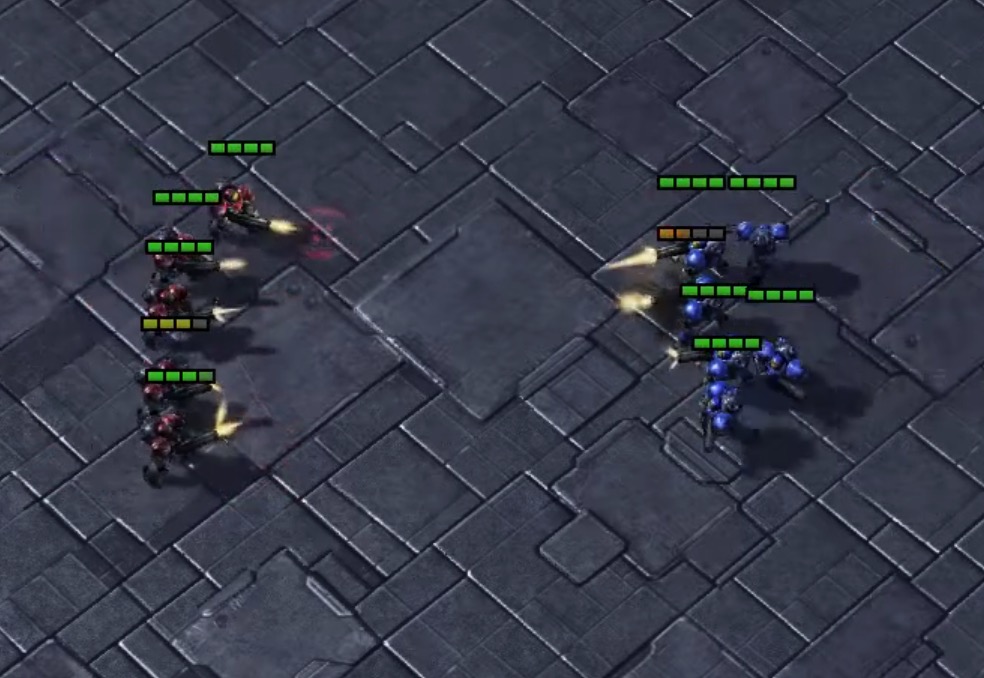}\label{img_5m_vs_6m}}
  \subfigure[MMM2]{\includegraphics[height=2.5cm,width=.23\textwidth]{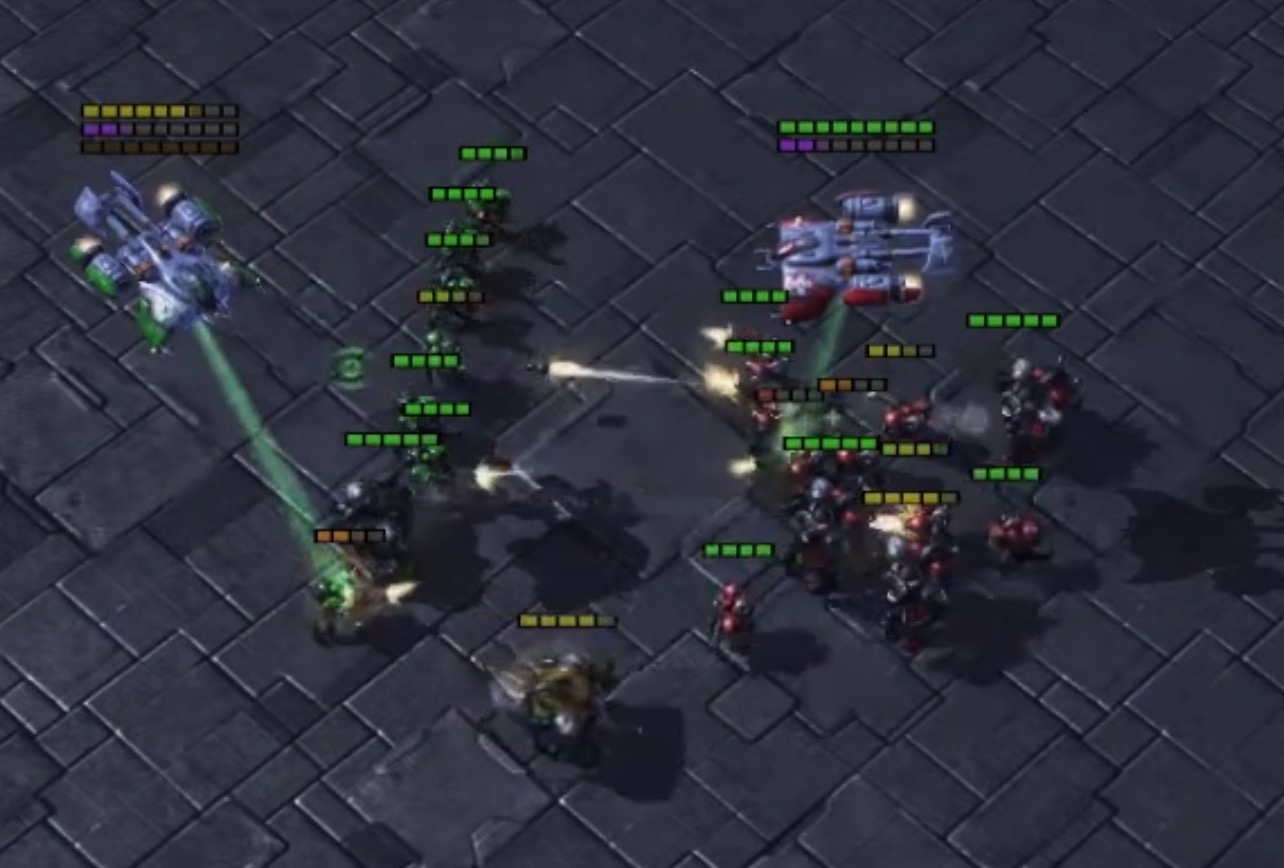}\label{img_MMM2}}
 \caption{SMAC scenarios: 5m\_vs\_6m and MMM2.}\label{sc_images}
\vspace{-0.3cm}
\end{center}
\end{figure}

\subsection{Experiment Setup}
\textbf{StarCraft II.} We consider SMAC~\citep{samvelyan19smac} benchmark, a challenging set of cooperative SCII maps for micromanagement, as our evaluation environments. The version SCII simulator is 4.10. We evaluate our method for every 10,000 training steps during training by running 32 episodes in which agents trained with our method battle with built-in game bots. 
We report the mean test won rate (percentage of episodes won of MARL agents) along with one standard deviation of test won rate (shaded in figures). We present the results of our method and baselines on 6 scenarios: 1c3s5z, MMM2, 5m\_vs\_6m, 8m\_vs\_9m, 10m\_vs\_11m and corridor.

\begin{table}[ht]
\vspace{-0.5cm}
\caption{Baselines}\label{table:baselines}
\centering
\begin{tabular}{{l}{l}}
 \hline
 \small{Baselines} & \small{Description} \\
 \hline
    \small{IQL} & \citet{tampuu2017multiagent}   \\
    \small{VDN} & \citet{sunehag2017value} \\
    \small{COMA} & \citet{foerster2017counterfactual} \\
    \small{QMIX} & \citet{rashid2018qmix} \\
    \small{QTRAN} & \citet{son2019qtran} \\
    \small{MAVEN} & \citet{mahajan2019maven} \\
    \small{LH-IQN} & \citet{lyu2020likelihood} \\ 
    \small{QR-MIX} & \citet{hu2020qr} \\
    \small{Qatten} & \citet{yang2020qatten}\\
 \hline
 \vspace{-0.5cm}
\end{tabular}
\end{table}
\textbf{Baselines and training.} We show baselines in Table \ref{table:baselines}. We implement our method on PyMARL~\citep{samvelyan19smac} and use 5 random seeds to train each method. We carry out experiments on NVIDIA Tesla V100 GPU 16G.

\subsection{Experimental Results}

\subsubsection{Comparison with baseline methods} 
\textbf{Comparison with value-based MARL baselines.} As depicted in Figure~\ref{res_1}, RMIX demonstrates substantial superiority over baselines in \emph{asymmetric} and \emph{homogeneous} scenarios. RMIX outperforms baselines in \emph{asymmetric homogeneous} scenarios: 5m\_vs\_6m  (\textbf{super hard} and \emph{asymmetric}), 8m\_vs\_9m  (\textbf{easy} and \emph{asymmetric}), 10m\_vs\_11m  (\textbf{easy} and \emph{asymmetric}) and corridor (\textbf{super hard} and \emph{asymmetric}). On 1c3s5z (\textbf{easy} and \emph{symmetric heterogeneous}) and MMM2 (\textbf{super hard}  and \emph{symmetric heterogeneous}), RMIX also shows leading performance over baselines.
RMIX improves coordination in a sample efficient way via risk-sensitive policies. Intuitively, for \emph{asymmetric} scenarios, agents can be easily defeated by the opponents. As a consequence, coordination between agents is cautious in order to win the game, and the cooperative strategies in these scenarios should avoid massive casualties in the starting stage of the game. Apparently, our risk-sensitive policy representation works better than vanilla expected Q values (QMIX, VDN and IQL, et al.) in evaluation. In \emph{heterogeneous} scenarios, action space and observation space are different among different types of agents, and methods with vanilla expected action value are inferior to RMIX. 

Interestingly, as illustrated in Figure \ref{fig:corridor_res}, RMIX also demonstrates leading exploration performance on \textbf{super hard} corridor scenario, where there is a narrow corridor connecting two separate rooms, and agents should learn to cooperatively combat the opponents to avoid being beaten by opponents with the divide-and-conquer strategy. RMIX outperforms MAVEN~\citep{mahajan2019maven}, which was originally proposed for tackling multi-agent exploration problems, both in sample efficiency and performance. After 5 million training steps, RMIX starts to converge while MAVEN starts to converge after over 7 million training steps.
\begin{figure}[ht]
% \vspace{-0.45cm}
\begin{center}
 \vspace{-0.3cm}
    \includegraphics[width=0.45\textwidth]{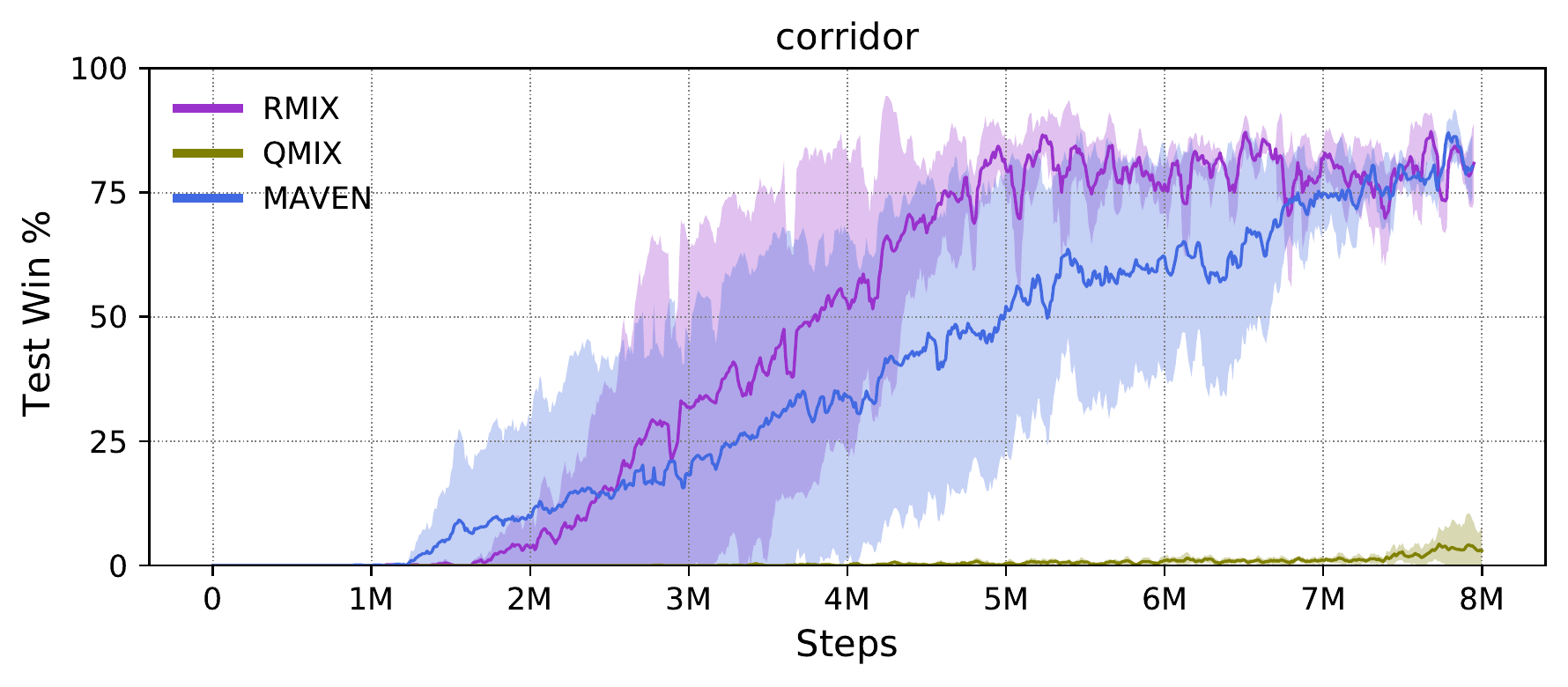}
    \vspace{-0.6cm}
    \caption{Test Winning rates for \textbf{super hard} scenario: corridor.}\label{fig:corridor_res}
 \vspace{-0.5cm}
 \end{center}
\end{figure}

\textbf{Comparison with Centralized Distributional MARL baseline.} We compare RMIX with QR-MIX, which learns a centralized return distribution while the policies of each agent are conventional Q values. As shown in Figure \ref{res_1}, RMIX shows leading performance and superior sample efficiency over QR-MIX on 6 scenarios. With a centralized return distribution, QR-MIX presents slightly better performance over other baselines on 1c3s5z, MMM2 and 5m\_vs\_6m. The results show that learning individual return distributions can better capture the randomness and thus improve the overall performance compared with learning central distribution which is only used during training while the individual returns are expected Q values.

\textbf{Comparison with risk-sensitive MARL baseline.} Although there are few practical methods on risk-sensitive multi-agent reinforcement learning, we also conduct experiments to compare our method with LH-IQN~\citep{lyu2020likelihood}, which is a risk-sensitive method built on implicit Quantile network. We can find that our method outperforms LH-IQN in many scenarios. LH-IQN performs poorly on StarCraft II scenarios as it is an independent learning MARL methods like IQL.

\begin{figure}[ht]
\vspace{-0.3cm}
    \begin{center}
        \includegraphics[width=0.5\textwidth]{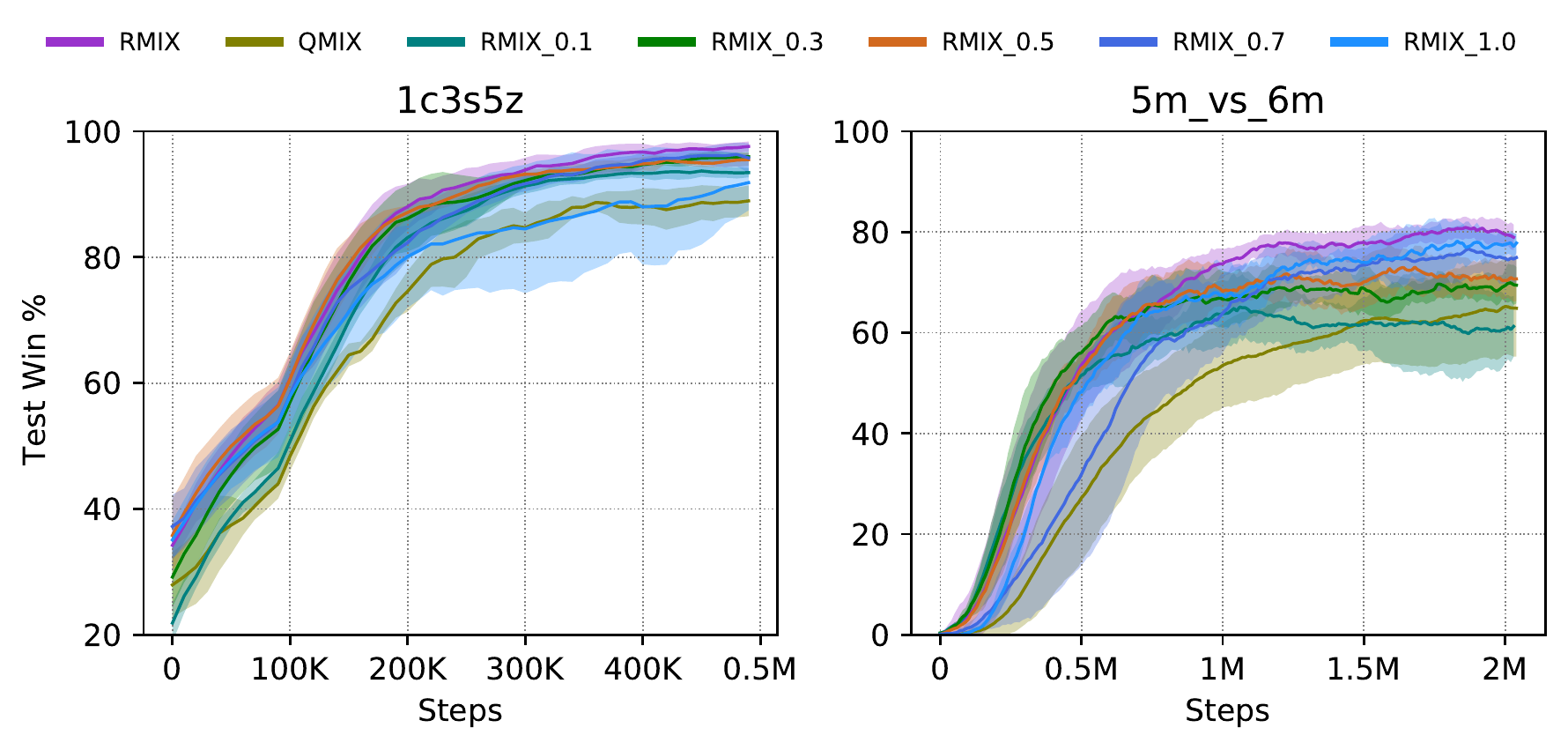}
        \vspace{-0.4cm}
        \caption{Ablation studies on various static risk levels.}\label{fig:ablation_1}
    \vspace{-0.3cm}
    \end{center}
\end{figure}
\begin{figure*}[ht]
% \vspace{-0.45cm}

\begin{center}
    \includegraphics[width=0.9\textwidth]{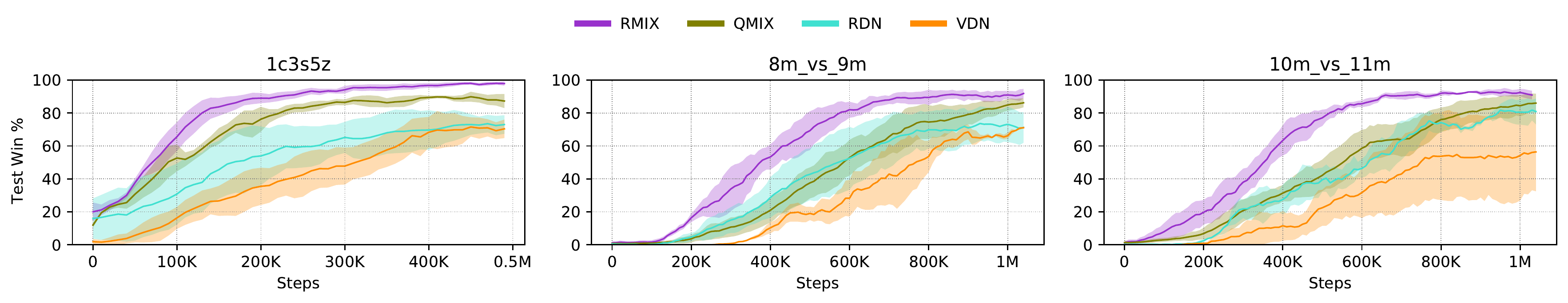}
    \vspace{-0.5cm}
    \caption{Test Winning rates of RMIX, RDN, VDN and QMIX.}\label{fig:ablation_study_vdn}
 \vspace{-0.3cm}
 \end{center}
\end{figure*}

\begin{figure*}[ht]
\begin{center}
    \includegraphics[width=0.75\textwidth]{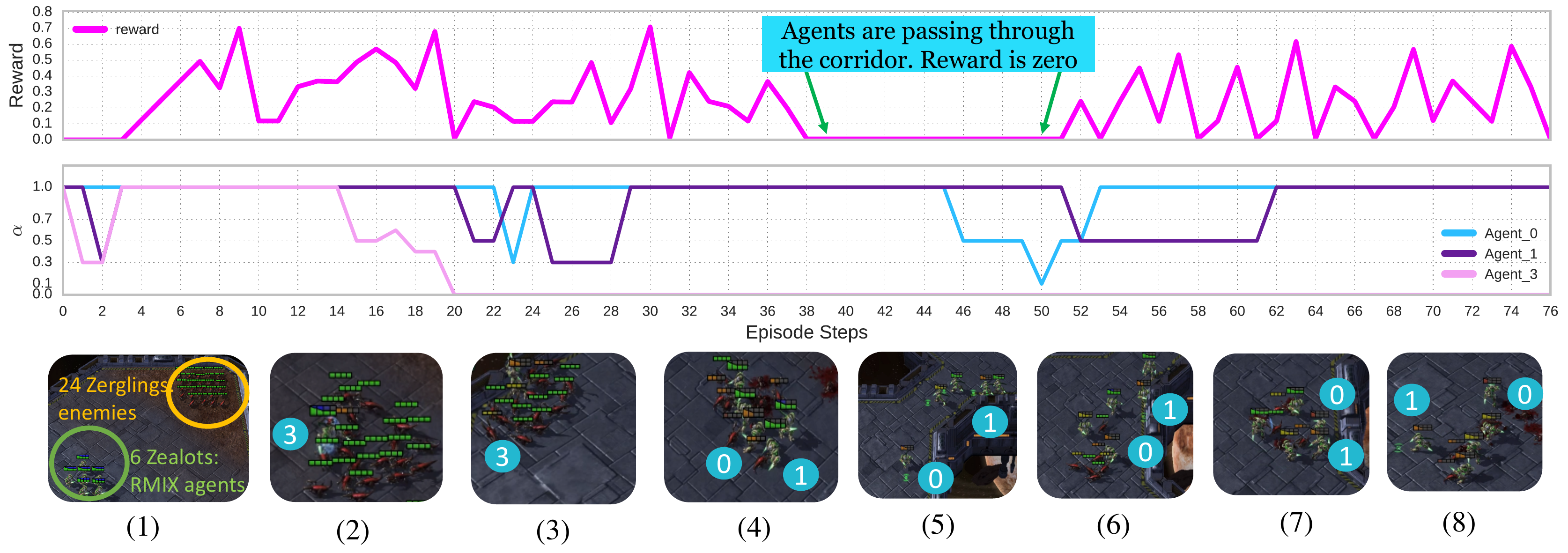}
    \vspace{-0.5cm}
        \begin{center}
            \caption{Results analysis of RMIX on corridor. There are 8 scenes as examples. Further discussion can be found in Sec. \ref{ResultsAnalysis}.}\label{fig:rmix_analysis}
        \end{center}
    \vspace{-0.8cm}
\end{center}
\end{figure*}

\subsubsection{Ablations}
RMIX mainly consists of two components: the CVaR policies and the risk level predictor. The CVaR policies are different vanilla Q values and the risk level predictor is proposed to model the temporal structure and as an $\alpha$-finding strategy for hyperparameter tuning. Our ablation studies serve to answer the following questions: \textbf{(a)} Can RMIX with static $\alpha$ also work? \textbf{(b)} Can the risk level predictor learn $\alpha$ values and fast learn a good policy compared with RMIX with static $\alpha$? \textbf{(c)} Can our framework be applied to other methods? 

To answer the above questions, we first conduct an ablation study by fixing the risk level in RMIX with the value of $\{0.1, 0.3, 0.5, 0.7, 1.0\}$ and compare with RMIX and QMIX on 1c3s5z (\textbf{easy} and \textit{symmetric heterogeneous}) and 5m\_vs\_6m (\textbf{super hard} and \textit{asymmetric}). As illustrated in Figure \ref{fig:ablation_1}, with static $\alpha$ values, RMIX is capable of learning good performance over QMIX, which demonstrates the benefits of learning risk-sensitive MARL policies in complex scenarios where the potential of loss should be taken into consideration in coordination. With risk level predictor, RMIX converges faster than RMIX with static $\alpha$ values, illustrating that agents have captured the temporal features of scenarios and possessed the $\alpha$ value tuning merit. 

To show that our proposed method is flexible in other mixing networks, we then apply additivity of individual CVaR values to represent the global CVaR value as $C^{\operatorname{tot}}(\boldsymbol{\tau}, \boldsymbol{u}) = C_{1} (\tau_{1}, u_{1}, \alpha_{1}) + \cdots + C_{N} (\tau_{N}, u_{N}, \alpha_{N})$. Following the training of RMIX, we name this method Risk Decomposition Network (RDN). We use experiment setup of VDN and train RDN on 3 SMAC scenarios. With CVaR values of actions as policies, RDN outperforms VDN on 10m\_vs\_11m and converges faster than VDN on 1c3s5z and 8m\_vs\_9m, as depicted in Figure \ref{fig:ablation_study_vdn}, demonstrating that our framework can be applied to VDN and outperforms VDN.

\subsection{Results Analysis} \label{ResultsAnalysis}

We are interested in finding if the risk level predictor can predict temporal risk levels. We use the trained model of RMIX and run the model to collect one episode data including game replay, states, actions, rewards and $\alpha$ values. As shown in Figure \ref{fig:rmix_analysis}, the first row shows the rewards of one episode and the second row shows the $\alpha$ value each agent predicts per time step. There are eight scenes in the third row to show how agents learn time-consistency $\alpha$ values. Scenes in Figure \ref{fig:rmix_analysis} are screenshots from the game replay.

We use trajectories of agent 0, 1 and 3 as examples in Figure \ref{fig:rmix_analysis}. Number in the circle indicates the index of the agent.  Scene (1): one episode starts, 6 Zealots consist of RMIX agents and 24 Zerglings compose enemies; Scene (2): in order to win the game, agent 3 draws the attention of enemies and goes to the other side of the battlefield. The $\alpha$ value is $0.3$ at step 2. Many enemies are chasing agent 3. The rest agents are combating with fewer number enemies; Scene (3): at step 14, agent 3 is at the corner of the battlefield the $\alpha$ is decreasing. As being outnumbered, agent 3 quickly dies and the $\alpha$ is zero; Scene (4)  agent 0 and 1 show similar $\alpha$ values as they are walking around and fighting with enemies from time step 22 to 30. Then agents kill enemies around and planing to go to the other side of the battlefield in order to win the game; Scene (5): agent 0 (low health value) is walking through the corridor alone to draw enemies to come over. To avoid being killed, $\alpha$ values are low (step 46-50) which means the policy is risk-averse. From step 38 and step 51, reward is zero; Scene (6): agent 0 is facing many enemies and luckily its teammates are coming to help. So, the $\alpha$ is increasing (step 50);  Scene (7): as there are many teammates around and the number of enemies is small, agents are going to win. Agent 0 and agent 1 walk outside the range of the observations of enemies in order to survive. The $\alpha$ values of agent 0 and agent 1 are 1 (risk-neutral); Scene (8): agents win the game. The video is available on this link: \url{https://youtu.be/5yBEdYUhysw}. Interestingly, the result shows emergent cooperation strategies between agents at different step during the episode, which demonstrates the superiority of RMIX.

\section{Conclusion and Future Work}\label{conclusion}
In this paper, we propose RMIX, a novel and practical MARL method with CVaR over the learned distributions of individuals' Q values as risk-sensitive policies for cooperative agents. Empirically, we show that our method outperforms baseline methods on many challenging StarCraft II tasks, reaching the state-of-the-art performance and demonstrating significantly enhanced coordination as well as high sample efficiency.

Risk-sensitive policy learning is vital for many real-world multi-agent applications especially in risky tasks, for example autopilot vehicles, military action, resource allocation, finance portfolio management and Internet of Things (IoT). For the future work, better risk measurement together with accurate spatial-temporal trajectory representation can be investigated. Also, learning to model other agents' risk levels and reach consensus with communication can be another direction for enhancing multi-agent coordination.

\bibliography{main}
\bibliographystyle{icml2021}

\appendix
\newpage
\onecolumn

% \appendix
% \addcontentsline{toc}{section}{Appendices}
% \section*{Appendices}

% \input{RMIX_ICML_2021/appendix}

\section{Proofs}\label{proof}
We present proofs of our propositions and theorem introduced in the main text. The numbers of proposition, theorem and equations are reused in restated propositions.

\begin{restatable}{assumption}{Assumption2}\label{assumption2}
The mean rewards are bounded in a known interval, i.e., $r \in [-R_{\operatorname{max}}, R_{\operatorname{max}}]$. 
\end{restatable}

This assumption means we can bound the absolute value of
the Q-values as $\vert Q_{sa} \vert \leq Q_{\operatorname{max}} = HR_{\operatorname{max}}$, where $H$ is the maximum time horizon length in episodic tasks. 

\PropContract*
\begin{proof}
We consider the sup-norm contraction,
\begin{equation}
\begin{aligned}
    & \left\vert \mathcal{T} C_{(1)}(\boldsymbol{s}, \boldsymbol{u}) - \mathcal{T} C_{(2)}(\boldsymbol{s}, \boldsymbol{u}) \right\vert \leq \gamma \left\lVert C_{(1)}(\boldsymbol{s}, \boldsymbol{u}) - C_{(2)}(\boldsymbol{s}, \boldsymbol{u}) \right\lVert_{\infty} \quad \forall \boldsymbol{s} \in \mathcal{S}, \boldsymbol{u} \in \mathcal{U}.
\end{aligned}
\end{equation}
The sup-norm is defined as $\left\lVert C \right\lVert_{\infty}=\sup_{\boldsymbol{s} \in \mathcal{S}, \boldsymbol{u} \in \mathcal{U}} \left\vert C(\boldsymbol{s}, \boldsymbol{u}) \right\vert$ and $C \in \mathbb{R}$.

In $\{C_{i}\}^{N}_{i=1}$, the risk level is fixed and can be considered implicit input. Given two different return distributions $Z_{(1)}$ and $Z_{(2)}$, we prove:
\begin{equation}\label{eq:cvar_contract}
\begin{aligned}
    \left\vert \mathcal{T} C_{(1)} - \mathcal{T} C_{(2)} \right\vert
    &\leq \max_{\boldsymbol{s}, \boldsymbol{u}}\left\vert \left[\mathcal{T}C_{(1)}\right](\boldsymbol{s}, \boldsymbol{u}) - \left[\mathcal{T}C_{(2)} \right](\boldsymbol{s}, \boldsymbol{u}) \right\vert \\
    &= \max_{\boldsymbol{s}, \boldsymbol{u}} \left\vert \gamma \sum\nolimits_{\boldsymbol{s}^{\prime}}\mathcal{P}(\boldsymbol{s}^{\prime}\vert\boldsymbol{s},\boldsymbol{u}) \left ( \max_{\boldsymbol{u}^{\prime}}C_{(1)}(\boldsymbol{s}^{\prime}, \boldsymbol{u}^{\prime})  -  \max_{\boldsymbol{u}^{\prime}}C_{(2)}(\boldsymbol{s}^{\prime}, \boldsymbol{u}^{\prime}) \right ) \right\vert \\
    &\leq \gamma \max_{\boldsymbol{s}^{\prime}}  \left\vert \max_{\boldsymbol{u}^{\prime}} C_{(1)}(\boldsymbol{s}^{\prime}, \boldsymbol{u}^{\prime}) - \max_{\boldsymbol{u}^{\prime}} C_{(2)}(\boldsymbol{s}^{\prime}, \boldsymbol{u}^{\prime}) \right\vert \\
    &\leq \gamma \max_{\boldsymbol{s}^{\prime}, \boldsymbol{u}^{\prime}} \left\lvert C_{(1)}(\boldsymbol{s}^{\prime}, \boldsymbol{u}^{\prime}) - C_{(2)}(\boldsymbol{s}^{\prime}, \boldsymbol{u}^{\prime}) \right\lvert \\
    &= \gamma \left\lVert C_{(1)} - C_{(2)} \right\lVert_{\infty}
\end{aligned}
\end{equation}
This further implies that
\begin{equation}
    \left\vert \mathcal{T}C_{(1)} - \mathcal{T}C_{(2)}\right\vert \leq \gamma \left\lVert C_{(1)} - C_{(2)} \right\lVert_{\infty} \quad \forall \boldsymbol{s} \in \mathcal{S}, \boldsymbol{u} \in \mathcal{U}.
\end{equation}
\end{proof}

With proposition \ref{propContract}, we can leverage the TD learning~\citep{sutton2018reinforcement} to compute the maximal CVaR value of each agent, thus leading to the maximal global CVaR value.

\begin{restatable}{proposition}{PropExpected}\label{PropExpected}
For any agent $i \in \{1, \dots, N\}$, $\exists \lambda(\tau_{i}, u_{i}) \in \left(0, 1\right]$, such that $C_{i} (\tau_{i}, u^{t-1}_{i}) = \lambda(\tau_{i}, u^{t-1}_{i}) \mathbb{E}\left[Z_{i} (\tau_{i}, u^{t-1}_{i}) \right]$.
\end{restatable}
\begin{proof}
We first provide that given a return distribution $Z$, return random variable $\mathscr{Z}$ and risk level $\alpha \in \mathcal{A}$, $\forall z$,  $\Pi_{\alpha}Z$ can be rewritten as $\mathbb{E}\left[\mathscr{Z} \vert \mathscr{Z} < z \right] < \mathbb{E}\left[ \mathscr{Z}\right]$. This can be easily proved by following \cite{fra}'s proof. Thus we can get $\Pi_{\alpha}Z < \mathbb{E}\left[Z \right]$, and there exists $\lambda_{(\tau_i, u^{t-1}_i)} \in (0, 1]$, which is a value of agent's trajectories, such that $\Pi_{\alpha}Z_i(\tau_i, u^{t-1}_i) = \lambda_{(\tau_i, u_i)} \mathbb{E}\left[Z_i(\tau_i, u^{t-1}_i) \right]$.
\end{proof}
Proposition \ref{PropExpected} implies that we can view the CVaR value as truncated values of Q values that are in the lower region of return distribution $Z_i(\tau_i, u_i)$. CVaR can be decomposed into two factors: $\lambda_{(\tau_i, u_i)}$ and $\mathbb{E}[Z_i(\tau_i, u_i)]$.

\PropOverEst*
\begin{proof}
We first consider the static risk level for each agent and the linear additivity mixer~\citep{sunehag2017value} of individual CVaR values.

(1) Given $\boldsymbol{\alpha}_1$ and $\boldsymbol{\alpha}_2$ and the mixer is the additivity function, we derive $\mathbb{E}[\Psi_{\alpha}]$ as follows
% \begin{equation}
\begin{align}
    \mathbb{E}[\Psi_{\alpha}] &\overset{\Delta}{=} \mathbb{E}\left[r + \gamma \max_{\boldsymbol{u}^{\prime}} C^{\operatorname{tot}}\left(\boldsymbol{s}^{\prime}, \boldsymbol{u}^{\prime}\right) - \left(r + \gamma \max_{\boldsymbol{u}^{\prime}} Q\left(\boldsymbol{s}^{\prime}, \boldsymbol{u}^{\prime}\right)\right)\right] \\ 
    &= \mathbb{E}\left[\gamma \left(\max_{\boldsymbol{u}^{\prime}} C^{\operatorname{tot}}\left(\boldsymbol{s}^{\prime}, \boldsymbol{u}^{\prime}\right)-\max _{\boldsymbol{u}^{\prime}} Q\left(\boldsymbol{s}^{\prime}, \boldsymbol{u}^{\prime}\right)\right)\right] \label{eq:global_max_to_individial_max} \\
    &= \mathbb{E}\left[\gamma \left(\sum\nolimits_{i=1}^{N} \max_{u_i} C_{i} (\tau_{i}, u^{t-1}_{i}, \alpha) -\sum\nolimits_{i=1}^{N} \max_{u_i} Q_{i} (\tau_{i}, u^{t-1}_{i}) \right)\right]
\end{align}
% \end{equation}
Following Proposition \ref{PropExpected} and Eqn. \ref{eq:cvar_risk}, given $\alpha_1 \leq \alpha_2$ and for any $i \in N$ and $u \in \mathcal{U}$, we can easily derive
\begin{equation}
    C_{i} (\tau_{i}, u^{t-1}_{i}, \alpha_1) \leq C_{i} (\tau_{i}, u^{t-1}_{i}, \alpha_2)
\end{equation}
Thus,
\begin{equation}
    \max_{u_i}C_{i} (\tau_{i}, u^{t-1}_{i}, \alpha_1) \leq \max_{u_i}C_{i} (\tau_{i}, u^{t-1}_{i}, \alpha_2)
\end{equation}
Then, we can get
\begin{align}
    \mathbb{E}[\Psi_{\alpha_1}] 
    &= \mathbb{E}\left[\gamma \left(\sum\nolimits_{i=1}^{N} \max_{u_i} C_{i} (\tau_{i}, u^{t-1}_{i}, \alpha_1) -\sum\nolimits_{i=1}^{N} \max_{u_i} Q_{i} (\tau_{i}, u^{t-1}_{i}) \right)\right] \\
    &\leq \mathbb{E}\left[\gamma \left(\sum\nolimits_{i=1}^{N} \max_{u_i} C_{i} (\tau_{i}, u^{t-1}_{i}, \alpha_2) -\sum\nolimits_{i=1}^{N} \max_{u_i} Q_{i} (\tau_{i}, u^{t-1}_{i}) \right)\right] \\
    &= \mathbb{E}[\Psi_{\alpha_2}]
\end{align}
Finally, we get $\mathbb{E}[\Psi_{\alpha_1}] \leq  \mathbb{E}[\Psi_{\alpha_2}]$.

Note that, it also applies when the mixer is the monotonic network by following the proof of Theorem 1 in \citep{rashid2020monotonic}. Here we present the proof in RMIX for the convenience of readers.

With monotonicity network $f_{\operatorname{m}}$, in RMIX, we have % \footnote{We omit the risk level $\alpha$ for notation brevity since it is generated by $\psi$ at each time-step.}

\begin{equation}
    C^{\operatorname{tot}} (\boldsymbol{s}, \boldsymbol{u}) = f_{\operatorname{m}}(C_{1}(\tau_{1}, u_{1}, \alpha_{1}), \dots, C_{n}(\tau_{n}, u_{n}, \alpha_{n}))
\end{equation}
Consequently, we have 
\begin{equation}
    C^{\operatorname{tot}} (\boldsymbol{s}, \{\arg\max_{u^{\prime}} C_{i}(\tau_{i}, u^{\prime}, \alpha_{i})\}_{i=1}^{N})
    = f_{\operatorname{m}}(\{\max_{u^{\prime}} C_{i}(\tau_{i}, u^{\prime}, \alpha_{i})\}_{i=1}^{N})
\end{equation}
By the monotonocity property of $f_{\operatorname{m}}$, we can easily derive that if $j \in \{1, \dots, N\}$, $u^{*}_j=\arg\max_{u_i}C_{j}(\tau_{j}, u^{t-1}_j, \alpha_{j})$, $\alpha_{j} \in (0, 1]$ is the risk level given the current return distributions and historical return distributions, and actions of other agents are not the best action, then we have
\begin{equation}
    f_{\operatorname{m}}(\{C_{j}(\tau_{j}, u_{j}, \alpha_{j})\}_{i=1}^{N}) \leq f_{\operatorname{m}}(\{C_{j}(\tau_{j}, u_{j}, \alpha_{j})\}_{i=1,i \neq j}^{N}, C_{j}(\tau_{j}, u^{*}_{j}, \alpha_{j})).
\end{equation}
So, for all agents, $\forall j \in \{1, \dots, N\}$, $u^{*}_j=\arg\max_{u_j} C_{j}(\tau_{j}, u^{t-1}_j, \alpha_{j})$, we have
\begin{align}
    f_{\operatorname{m}}(\{C_{j}(\tau_{i}, u_{i}, \alpha_{i})\}_{i=1}^{N}) &\leq f_{\operatorname{m}}(\{C_{j}(\tau_{j}, u_{j}, \alpha_{j})\}_{i=0,i \neq j}^{n-1}, C_{j}(\tau_{j}, u^{*}_{j})) \\
    &\leq f_{\operatorname{m}}(\{C_{i}(\tau_{i}, u^{*}_{i}, \alpha_{i})\}_{i=1}^{N}) \\
    &= \max_{\{u_{i}\}_{i=1}^{N}} f_{\operatorname{m}}(\{C_{i}(\tau_{i}, u_{i}, \alpha_{i})\}_{i=1}^{N}).
\end{align}
Therefore, we can get
\begin{equation}\label{eq:mono_max}
    \max_{\{u_{i}\}_{i=1}^{N}} f_{\operatorname{m}} (\{ C_{i} (\tau_{i}, u_i, \alpha_{i})\}_{i=1}^{N}) = \max_{\boldsymbol{u}} C^{\operatorname{tot}}(\boldsymbol{s}, \boldsymbol{u}),
\end{equation}
which implies
\begin{equation}
\max_{\boldsymbol{u}} C^{\operatorname{tot}}(\boldsymbol{s}, \boldsymbol{u}) = C^{\operatorname{tot}} (\boldsymbol{s}, \{\arg\max_{u^{\prime}} C_{i}(\tau_{i}, u^{\prime})\}_{i=1}^{N}).
\end{equation}
The Eqn. \ref{eq:mono_max} can be used in Eqn. \ref{eq:global_max_to_individial_max} to derive the about results.

(2) By following Proposition \ref{PropExpected}. We can get
\begin{align}
        \mathbb{E}[\Psi_{\boldsymbol{\alpha}}] 
        &= \mathbb{E}\left[\gamma \left(\sum\nolimits_{i=1}^{N} \max_{u_i} C_{i} (\tau_{i}, u^{t-1}_{i}, \alpha_i) -\sum\nolimits_{i=1}^{N} \max_{u_i} Q_{i} (\tau_{i}, u^{t-1}_{i}) \right)\right] \\
        &= \mathbb{E}\left[\gamma \left(\sum\nolimits_{i=1}^{N} \max_{u_i} C_{i} (\tau_{i}, u^{t-1}_{i}, \alpha_2) -\sum\nolimits_{i=1}^{N} \max_{u_i} Q_{i} (\tau_{i}, u^{t-1}_{i}) \right)\right] \\
        &=\mathbb{E}\left[\gamma \left(\sum\nolimits_{i=1}^{N} \max_{u_i}\lambda_{(\tau_i, u^{t-1}_i)} Q_{i} (\tau_{i}, u^{t-1}_{i}) -\sum\nolimits_{i=1}^{N} \max_{u_i} Q_{i} (\tau_{i}, u^{t-1}_{i}) \right)\right] \\
        &=\mathbb{E}\left[\gamma \left(\sum\nolimits_{i=1}^{N} \max_{u_i}\left((\lambda_{(\tau_i, u^{t-1}_i)}-1) Q_{i}(\tau_{i}, u^{t-1}_{i})\right)  \right)\right]
\end{align}
We can get that $\exists \alpha_i \in (0, 1]$ and $i \in \{1, \dots,  N\}$, $\mathbb{E}[\Psi_{\alpha}] < 0$.
\end{proof}

\newpage
\section{Pseudo Code of RMIX}\label{AppendixAddIntroRMIX}

\begin{algorithm}
\begin{algorithmic}[1]
\REQUIRE {$K$, $\gamma$;}
\REQUIRE Initialize parameters $\theta$ of the network of agent, risk operator and monotonic mixing network;\\
\REQUIRE Initialize parameters $\bar{\theta}$ of the target network of agent, risk operator and monotonic mixing network;\\
\REQUIRE Initialize replay buffer $\mathcal{D}$;\\
\FOR{$e \in \{1, \dots, \text{MAX\_EPISODE}\}$}
    \STATE Start a new episode;\\
	\WHILE {\text{EPISODE\_IS\_NOT\_TEMINATED}}
	    \STATE Get the global state $\boldsymbol{s}^t$;\\
	    \FORALL{agent $i \in \{1, \dots, N\}$}
	        \STATE Get observation $o^{t}_{i}$ from the environment; \label{alg:startCal}
	        \STATE Get action of last step $u^{t-1}_{i}$ from the environment;
	        \STATE Estimate the local return distribution $Z^{t}_i (o^{t}_{i}, u^{t-1}_{i})$;
	        \STATE Predict the risk level $\alpha_i$ (Eqn. \ref{eq:alpha_cal});  %: $\arg\max_{k} \left\{ \frac{\exp \left(\left\langle  f_{\mathrm{emb}}\left(Z_{i}\right)^{k}, \phi_{i}^{k}\right\rangle\right)}{\sum_{k^{\prime}=0}^{K-1} \exp \left(\left\langle f_{\mathrm{emb}}\left(Z_{i}\right)^{k^{\prime}}, \phi_{i}^{k^{\prime}}\right\rangle\right)} \right\}_{k} / K, k\in \{1, ...., K\}$\;
	        \STATE Calculate CVaR values $C_{i}^{t}\left(o^{t}_{i}, u^{t-1}_{i}, \alpha_{i}\right)$ (Eqn. \ref{eq:cvar_risk}); \label{alg:endCal}
	        \STATE Get the action $u^{t}_{i}=\arg\max_{u^{t}} C_{i}^{t}\left(o^{t}_{i}, u^{t-1}_{i}, \alpha_{i}\right)$;
	    \ENDFOR
	    
	    \STATE Concatenate $u^{t}_{i}$, $i$ $\in$ $[1, .., N]$ into $\mathbf{u}_{t}$;
	    \STATE Execute $\boldsymbol{u}^{t}_{i}$ into environment;
	    \STATE Receive global reward $r^{t}$ and observe a new state $\boldsymbol{s}^{\prime}$;
	    \STATE Store $(\boldsymbol{s}^{t}, \{o^{t}_{i}\}_{i \in [1, ..., N]}, \boldsymbol{u}^{t}, r^{t}, \boldsymbol{s}^{\prime})$ in replay buffer $\mathcal{D}$;

        \IF{UPDATE}
            \STATE Sample a min-batch $\mathcal{D}^{\prime}$ from replay buffer $\mathcal{D}$;
            \STATE For each sample in $\mathcal{D}^{\prime}$, calculate CVaR value $C_{i}$ by following steps in line 6-11;
            \STATE Concatenate CVaR values $\{[C^{1}_{1}, \dots, C^{1}_{N}]_{1}, \dots, [C^{\vert\mathcal{D}^{\prime}\vert}_{1}, \dots, C^{\vert\mathcal{D}^{\prime}\vert}_{N}]_{\vert\mathcal{D}^{\prime}\vert}\}$;
            \STATE For each $[C^{j}_{1}, \dots, C^{j}_{N}]_{0, j \in [1, \dots, \vert\mathcal{D}^{\prime}\vert]}$, calculate $C^{\operatorname{tot}}_{j}$ in the mixing network;
            % \STATE Calculate the target value ${y}^{\operatorname{tot}}=\left(r_{t}+\gamma \max_{\boldsymbol{u}^{\prime}} C_{\bar{\theta}}^{\operatorname{tot}}\right)$\;
            % \STATE Calculate the TD loss $\mathcal{L}_{\Pi}(\theta) := \mathbb{E}_{\mathcal{D}^{\prime} \sim \mathcal{D} }\left[({y}^{\operatorname{tot}} -  C^{\operatorname{tot}})^{2}\right]$\;
            \STATE Update $\theta$ by minimizing the TD loss (Eqn. \ref{eq:td_loss});
            \IF{UPDATE}
                \STATE Update the local return via QR loss (Eqn.\ref{eq:qr_loss});
            \ENDIF
            \STATE Update $\bar{\theta}$: $\bar{\theta} \leftarrow \theta$;
        \ENDIF
    \ENDWHILE
\ENDFOR
% \STATE Return $\theta$\;
\end{algorithmic}
\caption{RMIX}\label{alg:rmix}
\end{algorithm}

\newpage
\section{SMAC Settings}\label{appendEnv}

SMAC benchmark is a challenging set of cooperative StarCraft II maps for micromanagement developed by~\cite{samvelyan19smac} built on DeepMind's PySC2~\citep{vinyals2017starcraft}. We introduce \textbf{states and observations}, \textbf{action space} and  \textbf{rewards} of SMAC, and \textbf{environmental settings of RMIX} below.

\textbf{States and Observations.} At each time step, agents receive local observations within their field of view, which contains information (distance, relative x, relative y, health, shield, and unit type) about the map within a circular area for both allied and enemy units and makes
the environment partially observable for each agent.
The global state is composed of the joint observations, which could be used during training. All features, both in the global state and in individual observations of agents, are normalized by their maximum values.

\textbf{Action Space.} actions are in the discrete space. Agents are allowed to make move[direction], attack[enemy id], stop and no-op. The no-op action is only for dead agents and it is the only legal action for them. Agents can only move in four directions: north, south, east, or west. 
The shooting range is set to for all agents. Having a larger sight range than a shooting range allows agents to make use of the move commands before starting to fire. The automatical built-in behavior of agents is also disabled for training.

\textbf{Rewards.} At each time step, the agents receive a joint reward equal to the total damage dealt on the enemy agents. In addition, agents receive a bonus of 10 points after killing each opponent, and 200 points after killing all opponents for winning the battle. The rewards are scaled so that the maximum cumulative reward achievable in each scenario is around 20.

\textbf{Environmental Settings of RMIX.} The difficulty level of the built-in game AI we use in our experiments is level 7 (very difficult) by default as many previous works did~\citep{rashid2018qmix,mahajan2019maven}. The scenarios used in Section \ref{experiments} are shown in Table \ref{someEnvs}. We present the table of all scenarios in SMAC in Table \ref{allEnvs} and the corresponding memory usage for training each scenario in Table \ref{memory}. The \emph{Ally Units} are agents trained by MARL methods and \emph{Enemy Units} are built-in game bots. For example, 5m\_vs\_6m indicates that the number of MARL agent is 5 while the number of the opponent is 6. The agent (unit) type is \textit{marine}\footnote{A type of unit (agent) in StarCraft II. Readers can refer to \url{https://liquipedia.net/starcraft2/Marine_(Legacy_of_the_Void)} for more information}. This asymmetric setting is hard for MARL methods.

\begin{table}[ht]
    \caption{SMAC Environments}\label{allEnvs}\label{someEnvs}
    \begin{center}
    \begin{tabular}{{c}{c}{c}{c}}
    \hline
        Name & Ally Units & Enemy Units & Type \\ \hline\hline
        % 3m & 3 Marines & 3 Marines & homogeneous \& symmetric \\ \hline
        % 8m & 8 Marines & 8 Marines & homogeneous \& symmetric \\ \hline
        % 25m & 25 Marines & 25 Marines & homogeneous \& symmetric \\ \hline
        % 2s3z & \thead{2 Stalkers \& \\ 3 Zealots} & \thead{2 Stalkers  \& \\ 3 Zealots} & heterogeneous \& symmetric \\ \hline
        % 3s5z & \thead{3 Stalkers \& \\  5 Zealots} & \thead{3 Stalkers \& \\ 5 Zealots} & heterogeneous \& symmetric \\\hline
        % MMM & \thead{1 Medivac, \\ 2 Marauders \& \\ 7 Marines} & \thead{1 Medivac, \\ 2 Marauders \& \\ 7 Marines} & heterogeneous \& symmetric \\ \hline
        5m\_vs\_6m & 5 Marines & 6 Marines & homogeneous \& asymmetric \\ \hline
        8m\_vs\_9m & 8 Marines & 9 Marines & homogeneous \& asymmetric \\ \hline
        10m\_vs\_11m & 10 Marines & 11 Marines & homogeneous \& asymmetric \\ \hline
        % 27m\_vs\_30m & 27 Marines & 30 Marines & homogeneous \& asymmetric \\ \hline
        % 3s5z\_vs\_3s6z & \thead{3 Stalkers \& \\ 5 Zealots} & \thead{3 Stalkers \& \\ 6 Zealots} & heterogeneous \& asymmetric \\ \hline
        MMM2 & \thead{1 Medivac, \\ 2 Marauders \& \\ 7 Marines} & \thead{1 Medivac, \\ 3 Marauders \& \\ 8 Marines} & heterogeneous \& asymmetric \\ \hline
        1c3s5z & \thead{1 Colossi \& \\ 3 Stalkers \& \\ 5 Zealots} & \thead{1 Colossi \& \\ 3 Stalkers \& \\ 5 Zealots} & heterogeneous \& symmetric \\ \hline 
        % 2m\_vs\_1z & 2 Marines & 1 Zealot & micro-trick: alternating fire \\ \hline
        % 2s\_vs\_1sc & 2 Stalkers & 1 Spine Crawler & micro-trick: alternating fire \\ \hline
        % 3s\_vs\_3z & 3 Stalkers & 3 Zealots & micro-trick: kiting \\ \hline
        % 3s\_vs\_4z & 3 Stalkers & 4 Zealots & micro-trick: kiting \\ \hline
        % 3s\_vs\_5z & 3 Stalkers & 5 Zealots & micro-trick: kiting \\ \hline
        % 6h\_vs\_8z & 6 Hydralisks & 8 Zealots & micro-trick: focus fire \\ \hline
        corridor & 6 Zealots & 24 Zerglings & micro-trick: wall off \\ \hline
        % bane\_vs\_bane & \thead{20 Zerglings \& \\ 4 Banelings} & \thead{20 Zerglings \& \\ 4 Banelings} & micro-trick: positioning \\ \hline
        % so\_many\_banelings & 7 Zealots & 32 Banelings & micro-trick: positioning \\ \hline
        % 2c\_vs\_64zg & 2 Colossi & 64 Zerglings & micro-trick: positioning \\ \hline
    \end{tabular}
    \end{center}
\end{table}

\newpage
\section{Training Details}\label{appendTrainDetails}
The baselines are listed in table \ref{baselines} as depicted below. To make a fair comparison, we use {\fontfamily{qcr}\selectfont episode} (single-process environment for training, compared with {\fontfamily{qcr}\selectfont parallel}) runner defined in PyMARL to run all methods.  The evaluation interval is $10,000$ for all methods. We use uniform probability to estimate $Z_i(\cdot, \cdot)$ for each agent. We use the other hyper parameters used for training in the original papers of all baselines. The metrics are calculated with a moving window size of 5. Experiments are carried out on NVIDIA Tesla V100 GPU 16G. We also provide memory usage of baselines (given the current size of the replay buffer) for training each scenario of SCII domain in SMAC. 
We use the same neural network architecture of agent used by QMIX~\citep{rashid2018qmix}. The trajectory embedding network $\phi_i$ is similar to the network of the agent. The last layer of the risk level predictor generates the local return distribution and shares the same weights with the last layer of the agent network.

The QR loss is minimized to periodically (empirically every 50 time steps) update the weights of the agent as simultaneously updating the local distribution and the whole network can impede training. The QR loss is used for updating the local distribution while the TD loss of centralized training is for the agent network weights learning and credit assignment. In the QR loss, the $C_i(\cdot, \cdot, \cdot)$ is considered as scalar value and gradients from $C_i(\cdot, \cdot, \cdot)$ are blocked. As the QR update relies on accurately estimation of the dummy reward $C_i$, we start to update $C_i$ when the test wining rate is over 35\%, which means the agents have grasped some strategies to win the game.

\begin{table}[ht]
\caption{Memory usage (given the current size of the replay buffer) for the training of each method (exclude COMA, which is an on-policy method without using replay buffer) on scenarios of SCII domain in SMAC.}\label{memory}
\centering
\begin{tabular}{{c}{c}}
 \hline
 Scenario & Memory Usage (GB) \\
 \hline
    % 3m & 2.7 \\
    % 2m\_vs\_1z & 2.8 \\
    5m\_vs\_6m & 3 \\
    % 2s\_vs\_1sc & 3 \\
    % 3s\_vs\_3z & 3 \\
    % 3s\_vs\_4z & 3.5 \\
    % 2s3z & 3.9 \\
    % 3s\_vs\_5z & 4 \\
    % 6h\_vs\_8z & 4.6 \\
    % 8m &  4.8 \\
    8m\_vs\_9m & 4.9 \\
    % 3s5z & 6.4 \\
    % so\_many\_banelings & 7 \\
    10m\_vs\_11m & 7.1 \\
    % 3s5z\_vs\_3s6z & 7.5 \\
    1c3s5z & 8.6 \\
    % MMM & 8.7 \\
    MMM2 & 10.8 \\
    % 2c\_vs\_64zg & 12.1 \\
    corridor & 14.4 \\
    % 25m & 27 \\
    % 27m\_vs\_30m & 39.5 \\
    % bane\_vs\_bane & 41 \\
 \hline
\end{tabular}
\end{table}

\begin{table}[ht]
\caption{Baseline algorithms}\label{baselines}
\centering
\begin{tabular}{{c}{c}}
 \hline
 Algorithms & Brief Description \\
 \hline
    IQL~\citep{tampuu2017multiagent} & Independent Q-learning \\
    VDN~\citep{sunehag2017value} & Value decomposition network \\
    COMA~\citep{foerster2017counterfactual} & Counterfactual Actor-critic \\
    QMIX~\citep{rashid2018qmix} & Monotonicity Value decomposition \\
    QTRAN~\citep{son2019qtran} & Value decomposition with linear affine transform \\
    MAVEN~\citep{mahajan2019maven} & MARL with variational method for exploration \\
    QR-MIX~\citep{hu2020qr} & MARL with Centralized Distributional Q \\
    LH-IQN~\citep{lyu2020likelihood} & Likelihood Hysteretic with IQN (independent learning) \\
    Qatten~\citep{yang2020qatten} & Multi-head Attention for the estimation of the $Q^{\operatorname{tot}}$  \\
 \hline
\end{tabular}
\end{table}

\begin{table}[ht]
\caption{Hyper-parameters}\label{hyparam}
\centering
\begin{tabular}{{c}{c}}
 \hline
 hyper-parameter & Value \\
 \hline
    Batch size & 32 \\
    Replay memory size &  5,000 \\
    RMIX Optimizer & Adam \\
    Learning rate (lr) & 5e-4 \\
    Critic lr & 5e-4 \\
    RMSProp alpha & 0.99\\
    RMSProp epsilon & 0.00001 \\
    Gradient norm clip & 10 \\
    Action-selector  &  $\epsilon$-greedy \\
    $\epsilon$-start & 1.0 \\
    $\epsilon$-finish & 0.05 \\
    $\epsilon$-anneal time & 50,000 steps \\
    Target update interval & 200 \\
    Evaluation interval & 10,000 \\
    M & 35 \\
    K & 10 \\
    Runner & {\fontfamily{qcr}\selectfont episode} \\
    Discount factor ($\gamma$) & 0.99 \\
    RNN hidden dim & 64 \\
 \hline
\end{tabular}
\end{table}

\newpage
\section{Additional Results}\label{appendResults}

\paragraph{Addtional Results.} We show the test return value in Figure \ref{fig:appendix_res_2} and RMIX outperforms baseline methods as well. As shown in Figure \ref{fig:appendix_res_3}, RDN shows convincing test return value over VDN on three scenarios.

\begin{figure*}[ht]
% \vspace{-0.45cm}
\begin{center}
    \includegraphics[width=1\textwidth]{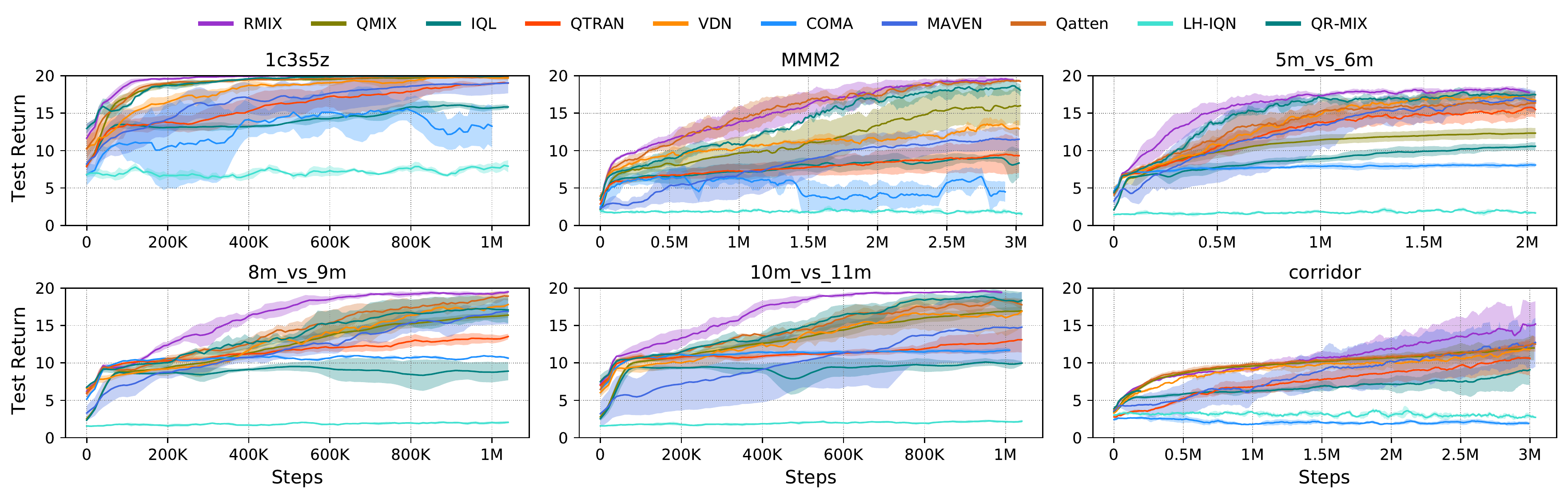}
    \vspace{-0.8cm}
    \caption{Test Return for six scenarios.}\label{fig:appendix_res_2}
 \vspace{-0.3cm}
 \end{center}
\end{figure*}

\begin{figure*}[ht]
% \vspace{-0.45cm}
\begin{center}
    \includegraphics[width=1\textwidth]{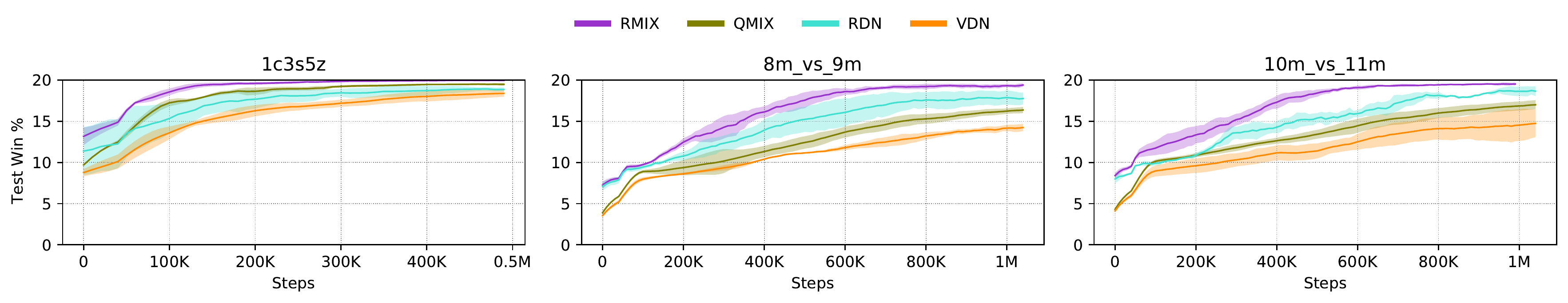}
    \vspace{-0.8cm}
    \caption{Test Return of RMIX, RDN, VDN, QMIX on three scenarios.}\label{fig:appendix_res_3}
 \vspace{-0.3cm}
 \end{center}
\end{figure*}

\end{document}